%% file: main-arxiv.tex
\documentclass[11pt]{article}

\title{Lazy OCO: Online Convex Optimization on a Switching Budget} 

\author{
Uri Sherman%
\thanks{Tel Aviv University; \texttt{urisherman@mail.tau.ac.il}.}
\and 
Tomer Koren%
\thanks{Tel Aviv University and Google Research, Tel Aviv; \texttt{tkoren@tauex.tau.ac.il}.}
}
\usepackage{header}

\usepackage[round]{natbib}
\usepackage{amsfonts}       
\usepackage{nicefrac}       
\usepackage{microtype}
\usepackage{amsmath,amssymb,amsthm}

\usepackage{graphicx}
\usepackage[colorlinks,allcolors=blue]{hyperref}

\usepackage{algorithmic,algorithm}
\usepackage{float}
\usepackage{mathtools,thmtools}
\usepackage[title]{appendix}
\usepackage{makecell}
\usepackage{ dsfont }
\usepackage[capitalize]{cleveref}
\usepackage{xspace}

\usepackage{multirow}
\usepackage[flushleft]{threeparttable}

\usepackage{enumitem}

\theoremstyle{plain}
\newtheorem{theorem}{Theorem}
\newtheorem{lemma}{Lemma}

\newtheorem{assumption}{Assumption}
\newtheorem{definition}{Definition}

\declaretheoremstyle[ 
        spaceabove=\topsep, 
        spacebelow=\topsep, 
        headfont=\normalfont\itshape,
        bodyfont=\normalfont,
        notefont=\normalfont\itshape,
        notebraces={}{},
        postheadspace=0.33em, 
        qed=$\square$, 
        headpunct={.},
    ]{proofstyle}
\declaretheorem[style=proofstyle,numbered=no,name=Proof]{proof}

\DeclarePairedDelimiterX{\infdiv}[2]{(}{)}{%
  #1\;\delimsize\|\;#2%
}
\DeclarePairedDelimiterX{\infdivb}[2]{\Big(}{\Big)}{%
  #1\;\delimsize\|\;#2%
}
\newcommand{\KL}{D_{KL}\infdiv}

\newcommand{\R}{\mathbb{R}}
\newcommand{\norm}[1]{\left\|#1\right\|}

\newcommand{\Exp}[1]{\E \left[ #1 \right]}
\newcommand{\Expp}[2]{\underset{#1}{\E} \left[ #2 \right]}

\newcommand{\T}{^\mathsf{T}}

\renewcommand{\Re}{\mathcal R}
\newcommand{\Sw}{\mathcal S}
\newcommand{\ESw}{\E \mathcal S_T}

\newcommand{\Alg}{\mathcal A}

\newcommand{\Ind}[1]{\mathds{1}\left\{ #1 \right\}}
\newcommand{\integ}[3]{\int_{#1} #2 \,\mathrm{d}#3}

\DeclareMathOperator*{\argmin}{arg\,min}

\DeclareMathOperator{\E}{\mathbb{E}}

\newcommand{\Dom}{W}
\newcommand{\dist}[1]{\mathcal #1}
\newcommand{\rv}[1]{\textbf #1}
\newcommand{\Proj}[1]{\Pi\left( #1 \right)}
\newcommand{\Projw}[1]{\Pi_\Dom\left( #1 \right)}

\newcommand{\Ber}{\mathrm{Ber}}

\newcommand{\TV}[1]{\norm{#1}_{TV}}
\newcommand{\cTV}{C_1}
\newcommand{\cTVv}{C_0}

\newcommand{\eqq}{\coloneqq}

\usepackage{xcolor}
\usepackage{soul}

\def \Ssoco {$S$-lazy\xspace}
\def \soco {lazy OCO\xspace}
\def \scost {switching-cost\xspace}
\def \sbudg {switching-budget\xspace}

\crefname{lem}{Lemma}{Lemma}

\hypersetup{
           breaklinks=true,   
}

\begin{document}

\maketitle

\input{doc_body_rectified.tex}

\subsection*{Acknowledgements}
We would like to express our gratitude to Naaman Agarwal, Karan Singh, Satyen Kale, and Abhradeep Guha Thakurta for recognizing an error in an earlier version of this manuscript and sharing with us their ideas leading to its resolution.
In addition, we thank Maayan Gal and Lior Ziv for pointing out an unnecessary log factor in an earlier version of this manuscript.
This work was partially supported by the Israeli Science Foundation (ISF) grant no.~2549/19, by the Len Blavatnik and the Blavatnik Family foundation, and by the Yandex Initiative in Machine Learning. 

\bibliographystyle{abbrvnat}
\bibliography{main}

\appendix

\input{doc_lb_sc_full}

\input{doc_appendix_rectified.tex}

\end{document}

%% file: doc_body_rectified.tex
\begin{abstract}%
We study a variant of online convex optimization where the player is permitted to switch decisions at most $S$ times in expectation throughout $T$ rounds. Similar problems have been addressed in prior work for the discrete decision set setting, and more recently in the continuous setting but only with an adaptive adversary. 
In this work, we aim to fill the gap and present computationally efficient algorithms in the more prevalent oblivious setting, establishing a regret bound of $O(T/S)$ for general convex losses and $\widetilde O(T/S^2)$ for strongly convex losses. In addition, for stochastic i.i.d.~losses, we present a simple algorithm that performs $\log T$ switches with only a multiplicative $\log T$ factor overhead in its regret in both the general and strongly convex settings.
In addition, for stochastic i.i.d.~losses, we present a simple algorithm that performs $\log T$ switches with only a constant factor overhead in its regret in the general convex setting, and a $\log T$ factor overhead in the strongly convex setting.
Finally, we complement our algorithms with lower bounds that match our upper bounds in some of the cases we consider.
\end{abstract}


\section{Introduction}

We study online convex optimization with limited switching. In the classical online convex optimization (OCO) problem, a player and an adversary engage in a $T$-round game, where in each round, 
the player chooses a decision $w_t \in \Dom \subseteq \R^d$, and the adversary responds with a loss function $f_t\colon \Dom \to \R$.
The losses $f_t$ are convex functions over $\Dom$ which is also convex and traditionally referred to as the decision set.
Each round incurs a loss of $f_t(w_t)$ against the player, whose objective is to minimize her cumulative loss. The performance of the player is then measured by her regret, defined as the difference between her cumulative loss and that of the best fixed decision in hindsight;
\[
   \sum_{t=1}^T f_t(w_t) - \min_{w\in \Dom} \sum_{t=1}^T f_t(w).
\]
This theoretical framework has found diverse applications in recent years, many of which benefit from player strategies that switch decisions sparingly. In adaptive network routing \citep{awerbuch2008online} switching decisions amounts to changing packet routes, which should be kept to a minimum as it may lead to severe networking problems (see, e.g., \citealp{feamster2014road}). 
When investing in the stock market, transactions may be associated with fixed commission costs, and thus trading strategies that change stock positions infrequently are of value.
As another example, \cite{geulen2010regret} approach online buffering by devising a low switching variant of the well known Multiplicative Weights algorithm.
In addition, recent applications of OCO in online reinforcement learning and control problems involve addressing the fact that changing policies introduces short-term penalties, and thus could benefit from keeping the number of policy switches to a minimum~%
(e.g., \citealp{cohen2018online,cohen2019learning,agarwal2019online,agarwal2019logarithmic,foster2020logarithmic}). 

This motivates the study of regret bounds achievable when we limit the number of decision switches the player is allowed to perform. When the limit is applied to the expected number of decision switches, we arrive at a variant of the standard model we shall refer to as \emph{lazy OCO}. In particular, we say that an OCO algorithm is \Ssoco if the expected number of switches it performs over $T$ rounds is less than $S$.
A closely related problem where the player is charged a fixed price $c>0$ per switch has been the focus of several works in the past, though mainly in the 
context of experts or multi-armed bandit problems~\citep{dekel2014bandits, geulen2010regret, altschuler2018online}.
It is not hard to see this problem, which we refer to as \scost OCO, is effectively equivalent to \soco. (We defer formal details to \cref{apdx:sccb}.)

Perhaps the most natural approach for this problem would be to divide the $T$ rounds into $S$ equally sized time-blocks, and treat the cumulative loss of each block as a single loss function. This effectively reduces the game to the standard unconstrained OCO setting with $S$ rounds, and a Lipschitz constant larger by a factor of $T/S$. This method has been termed ``blocking argument'' and dates back at least to \cite{merhav2002sequential} who used it to obtain an $O(T/\sqrt S)$ regret bound on the prediction with expert advice problem. 
It is not hard to see that this strategy also yields $O(T/\sqrt S)$ regret in the general convex setting. Recently, \cite{chen2019minimax} prove that this is in fact optimal against an \emph{adaptive} adversary with linear losses. 
However, in the oblivious adversary setting, stronger results may be achieved owed to the power of randomization. For example, several works \citep{kalai2005efficient,geulen2010regret,devroye2013prediction} have obtained a stronger $O(T/S)$ bound for the experts problem by employing randomized player strategies.
In the general convex setting results have been more scarce, though the same bound has been achieved by \cite{anava2015online}, who adapt the method of \cite{geulen2010regret} to the continuous online optimization setup.

In this work, we aim to further develop our understanding of \soco by addressing a number of questions that have remained open. 
First, the algorithm in \cite{anava2015online} obtains a regret bound of $\widetilde O(\sqrt {d T} + d T/S)$ where $d$ denotes the dimension of the decision set, and thus exhibits dimension dependence even in the case that $T$ switches are permitted.
Second, to the best of our knowledge, no results have been established in the strongly convex \soco setting, neither for adaptive adversaries nor for oblivious ones. 
Finally, in terms of dependence on the number of switches $S$, it is not clear a priori whether the additive $O(T/S)$ term in the regret is essential, and in particular, whether an $O(\sqrt T)$ regret bound may be obtained with $S=o(\sqrt T)$ switches.
This work aims to fill these gaps and obtain a more coherent understanding of the \soco problem.

\subsection{Our contributions}

We make the following contributions:
\begin{itemize}
    \item \textit{Regret upper bounds.} We present a computationally efficient \Ssoco algorithm, achieving an $O(\sqrt T + d T/S)$ regret bound for general convex losses, and an $\widetilde O(d^2 T/S^2)$ bound for strongly convex losses.
    Compared to the algorithm of \cite{anava2015online}, we remove the $\sqrt d$ factor from the $\sqrt T$ term, and our algorithm further extends to the strongly convex case where it obtains improved regret bounds, which does not seem to be the case for the algorithm of \cite{anava2015online}.%
    \footnote{In fact, a closer look into the regret analysis of \cite{anava2015online} reveals that it does not at all exploit the convexity of the losses: by taking a discretization of the decision set and using, e.g., the Shrinking-Dartboard algorithm of \cite{geulen2010regret}, one would obtain essentially the same regret guarantee (albeit not in polynomial time).} 

    \item \textit{Regret lower bounds.} For the general convex case, we prove an $\Omega(T/S)$ lower bound for the regret of any \Ssoco algorithm, matching our upper bound in this setting in terms of dependence on $T$ and $S$. In the strongly convex adaptive setting, we prove an $\Omega(T/S)$ lower bound, which matches up to a logarithmic factor the $O((T/S)\log S)$ upper bound obtained by a straightforward application of the blocking technique (see \cref{sec:lb:sc_adaptive} for details). Finally, for the strongly convex oblivious setting, we prove an $\Omega(T/S^2)$ lower bound for a certain family of algorithms, as discussed in \cref{sec:lb_statements}.
    
    \item \textit{Regret bounds for stochastic i.i.d.~losses.} For the special case of a stochastic i.i.d.~adversary, we present an algorithm that performs $O(\log T)$ switches while introducing only a constant multiplicative $O(\log T)$ factor overhead in the regret bound compared to unrestricted OCO in the general convex case, and only an extra logarithmic factor in the strongly convex case.
\end{itemize}

\cref{table:sotas} lists our contributions compared to the relevant state-of-the-art bounds. Our upper bounds are discussed in \cref{sec:ub,sec:stochastic}, and our lower bounds in \cref{sec:lb_statements}.

\paragraph{Addendum.}

Following the initial publication of this work in COLT'21~\citep{shermancolt2021lazy}, an error in the arguments given there was brought to our attention by the authors of \citet{agarwal2023differentially} through private correspondence.
Specifically, the random variable $p_t - \nabla f_t(w_t)$ in line~6 of Algorithm 2 in \citep{shermancolt2021lazy} is not normally distributed (due to dependence between $p_t$ and $w_t$), while Lemma 3 (through the use of Lemmas 1 and 2) requires it to be so.
The current version of the manuscript presents a corrected version (see \cref{alg:lftprl} and \cref{lem:tv_decisions}), by essentially replicating the approach proposed by \citet{agarwal2023differentially}.
The key aspect in this approach is to setup the algorithm in such a way that would allow working with the density functions of the decision variables $w_t$ directly. This is elegantly accomplished by the use of barrier functions and the change of variables formula, which ultimately allow coupling over consecutive decisions in spite of the dependence between $p_t$ and $w_t$.

For completeness, the bounds for Lazy OCO obtained in the recent work of \citet{agarwal2023differentially} were added to \cref{table:sotas} below; since ultimately our corrected algorithm is nearly identical to theirs,\footnote{Essentially, the only difference between the two algorithms is in the implementation of the coupling mechanism used to correlate sampling of consecutive decisions.} our regret bounds match theirs up to small logarithmic factors.
We also note that \citet{agarwal2023differentially} additionally give improved bounds for the case where the loss functions are in the form of generalized linear models (GLMs), which we do not reproduce here.

\subsection{Key ideas and techniques}

Our starting point for designing lazy algorithms in the (adversarial, oblivious) OCO setup is the general idea present in Follow-the-Lazy-Leader (FLL) algorithm of \cite{kalai2005efficient}, where perturbations are introduced for obfuscating small changes in the player's (unperturbed) decisions.
Then, the perturbations may be correlated in such a way that preserves the marginal distribution of decisions, and at the same time have sufficient overlap in total variation, which allows for the player to avoid switching altogether across several consecutive rounds.
However, unlike \cite{kalai2005efficient} who study the linear case, we are interested in the general convex and strongly convex settings, which pose a number of additional challenges.

First, the subset of perturbed objectives cannot be fixed in advance and must be determined per round in a dynamical fashion during execution of the algorithm. 
This is because the particular perturbation that leads to the same decision being used across rounds depends on the loss sequence in a complex way that mandates ad-hoc coupling between consecutive decisions. 
%
Even further, the unperturbed decisions need to be stabilized so that consecutive decision distributions overlap sufficiently in total variation. To that end, a regularization component is added to obtain the desired relation between regret and the number of decision switches.
Thus, unlike Follow-the-Perturbed-Leader-type algorithms that introduce perturbations for promoting stability, we draw our stability properties from a regularization component while using perturbations only for inducing proximity in total variation, which in turn allows the algorithm to resample decisions less frequently.

Finally, our regret bound for strongly convex losses makes use of two additional ideas that were key in achieving the improved dependence on the switches parameter $S$. 
First, perhaps surprisingly, the perturbation scale has to be \emph{increased} at a rate that is in accordance with the increasing curvature of the per round minimization objective (despite the fact that the unperturbed decision actually becomes more and more stable with time).
The second and more crucial observation is that the regret penalty introduced by perturbations on top of the hypothetical ``be-the-leader'' strategy can be bounded much more efficiently for strongly convex losses:
our analysis reveals that this penalty depends on the distance between the deterministic, unperturbed minimizer and the perturbed random one;
crucially, with strong convexity, this distance shrinks rapidly with the number of steps at a rate that compensates for the increased perturbation scale.

\begin{table}[t]
\bgroup \def\arraystretch{2.1}
\begin{center}
    \begin{threeparttable}    
    \caption{$S$-\soco bounds, omitting log terms (except for the i.i.d.~case) and factors other than $S$ and $T$. Our contributions are in boldface.
    }
        \small
       \begin{tabular}{ c | c | c | c}
                 \textsc{Setting}
               & \textsc{Adversary}
               & \textsc{Lower Bound}
               & \textsc{Upper Bound}

           \\ \hline\hline
           \makecell{
            Experts 
           }
               & Oblivious
               & \makecell{$ 
                    T / S $
                   \\ 
                   \cite{geulen2010regret}
               }
               & \makecell{$ T / S $ \tnote{a}
                   \\
                   \cite{kalai2005efficient} \tnote{b}
               }
                 
           \\ \hline\hline
           \multirow{6}{*}{ 
                OCO
            }
               & Adaptive
               & \makecell{$T / \sqrt S $
                   \\
                   \cite{chen2019minimax}
               }
               & \makecell{$ T / \sqrt S $
                   \\
                   \cite{chen2019minimax}
               }
           
           \\ \cline{2-4}
                   & Oblivious
                   & \makecell{
                       $ \boldsymbol{ T / S } $
                   }
                   & \makecell{\\
                        $ \boldsymbol{ T / S } $ \tnote{a}
                       \\
                   \textbf{this work},
                   \\
                   \citet{agarwal2023differentially}
                   \\
                   \citet{anava2015online}
                   \\ 
            \\
                   }
            \\ \cline{2-4}
                   & i.i.d.
                   & \makecell{
                       $ \sqrt T $
                       \\
                   }
                   & \makecell{
                        $ \boldsymbol{ \sqrt T  } $  \tnote{c}
                       \\
                   }
    
           \\ \hline\hline
           \multirow{5}{*}{\makecell{
                OCO 
                \\
                Strongly Convex
           }}
           
               & Adaptive
               & \makecell{
                   $ \boldsymbol{ T / S } $
                   \\
               }
               & \makecell{$(T/S)$ 
                \\
               }
    
           \\ \cline{2-4}
                   & Oblivious
                   & \makecell{
                         $ \boldsymbol{T / S^2 }  $ \tnote{d}
                   }
                   & \makecell{ \\
                   $ \boldsymbol{T/S^2 } $ 
                   \\
                   \textbf{this work}
                   \\
                  \citet{agarwal2023differentially}\tnote{$\dagger$}
            \\
            \\
                   }
                   
            \\ \cline{2-4}
                   & i.i.d.
                   & \makecell{
                         $ \log T $
                   }
                   & \makecell{ $ \boldsymbol{\log^2 T }  $ \tnote{c}
                    \\
                   }
            \\ \hline
       \end{tabular}
       \label{table:sotas}
       \vspace{1ex}
       \begin{tablenotes}\footnotesize
            \item[a] For $S = O \big(\sqrt T \big)$.
            \item[b] Also \cite{geulen2010regret, devroye2013prediction, altschuler2018online}.
            \item[c] For $S\geq 1 + \log T$.
            \item[d] For a certain restricted class of algorithms. See \cref{sec:lb_statements} for details.
        \end{tablenotes}
    \end{threeparttable}
\end{center}
\egroup
\end{table}

\subsection{Additional related work}

Prior work on low switching strategies in online learning has been mostly concerned with the \scost perspective.
All bounds we present here pertain to algorithms with an expected number of switches bounded by $S$, so that they are easily comparable. For completeness, their equivalent original \scost forms can be found in \cref{apdx:sccb}.

\paragraph{Experts.} 

The experts problem with switching costs has been extensively studied, giving rise to several algorithms
such as FLL \citep{kalai2005efficient}, Shrinking-Dartboard \citep{geulen2010regret} and Perturbation-Random-Walk \citep{devroye2013prediction}, all of which achieve $O(T / S)$ regret known to be optimal due to a matching lower bound of \cite{geulen2010regret}. 
Recently, \cite{altschuler2018online} study experts and multi-armed bandits in the setting where the player is given a \emph{hard cap} on the number of switches she is allowed (see \cref{apdx:sccb} for a discussion of this variant of the model); they develop a framework converting Follow-the-Perturbed-Leader (FPL) type algorithms that work in expectation to ones with high probability guarantees, and leverage this result to achieve an upper bound of $\widetilde O(T / S)$ for $S = O(\sqrt T)$, 
shown to be tight.

\paragraph{Multi-armed bandits.} 

Unlike experts, the multi-armed bandit problem has proved to exhibit a more significant dependence on the number of switches, setting apart \scost regret from the standard unconstrained setting. 
An $O( T / \sqrt{ S } )$ upper bound was obtained by a blocking argument~\citep{arora2012online} applied to the EXP3 algorithm~\citep{auer2002nonstochastic}. 
A matching $\widetilde \Omega(T / \sqrt S )$ lower bound was proved by~\cite{dekel2014bandits}. In the case of stochastic i.i.d.~losses, \cite{cesa2013online} present an $\widetilde O(\sqrt T)$ algorithm for multi-armed bandits that performs $O(\log \log T)$ switches.

\paragraph{Online convex optimization.}

To the best of our knowledge, \cite{anava2015online} establish the first and only $O(T / S)$ upper bound in the general convex setting with an oblivious adversary, albeit with a computationally intensive algorithm whose running time is bounded by a high-degree polynomial in the dimension. 
More recently, \cite{chen2019minimax} study \soco in the general convex setting with an adaptive adversary and prove a tight $\Theta(T / \sqrt S)$ result.
Our work is thus complementary to theirs as we study \soco in the oblivious setting, where stronger upper bounds turn out to be possible.
Also relevant to our work is the paper of \cite{jaghargh2019consistent}, who propose a Poisson process based algorithm for both general and strongly convex losses, although their results are suboptimal compared to those presented here.

\paragraph{Movement costs.} 

Also related to lazy OCO is the study of movement costs in online learning, where the player pays a switching cost proportional to the distance between consecutive decisions. This variant was studied in the context of multi-armed bandits~\citep{koren2017bandits,koren2017multi}, and is at the core of the well known metrical-task-systems (MTS) framework in competitive analysis~\citep{borodin1992optimal, borodin2005online}. 
In particular, the continuous variant of MTS has been the subject of several works both in the low dimensional setting \citep{bansal20152, antoniadis2017tight}, and in the high dimensional setting where it has been recently termed \emph{smoothed OCO} \citep{chen2018smoothed, goel2019beyond, shi2020online}. 
MTS differs from lazy OCO in a number of important ways; we refer to \cite{blum2000line, buchbinder2012unified, andrew2013tale} for an extensive discussion of the relations between competitive analysis and regret minimization.

\paragraph{Correlated sampling.} 

The algorithms we present are based on a lazy sampling procedure for sampling from maximal couplings (see \cref{sec:max_couplings}). This procedure bears similarity to the well-known correlated sampling problem \citep{broder1997resemblance, kleinberg2002approximation}, where two players are given two probability distributions and are required to produce samples with minimal disagreement probability. 
As the players are not allowed to communicate, this problem is crucially different than sampling from maximal couplings; see \cite{bavarian2020optimality} for a more elaborate discussion.

\paragraph{Differentially-private online learning.}

It has recently been observed that online learning with switching constraints is strongly related to differentially-private online learning~\citep{asi2023private,agarwal2023differentially}.
Differential Privacy is concerned with learning mechanisms that produce outputs 
that reveal little about any individual data point used to train them (see \citealp{dwork2014algorithmic}). In the context of online learning, and more specifically OCO, this boils down to designing algorithms with the property that if a single loss function were changed, the sequence of decisions produced by the algorithm would not change by much. 

There has been a long line of work studying differentially private prediction from expert advice and online learning more generally (e.g., \citealp{jain2012differentially,guha2013nearly,jain2014near,agarwal2017price,kairouz2021practical,asi2023near,kaplan2023differentially}).
Recently, \citet{asi2023private} and subsequently \citet{agarwal2023differentially} used low-switching online learning as a means to design differentially private algorithms.
Informally, the key idea behind this approach is that information about the loss sequence is leaked only when the online learning algorithm changes its decision; hence, an online algorithm can be transformed into a privacy-preserving one by limiting the number of switches it performs.

\section {Preliminaries}
\label{sec:prelim}

We start by giving a precise definition of our model and describe techniques and basic tools we use.

\subsection{Problem setup}
We describe the setting of \soco, within which we develop all results presented in the paper. In this setting, an oblivious adversary chooses convex Lipschitz loss functions $f_t\colon \Dom \to \R$ over a convex domain $\Dom \subseteq \R^d$. Throughout the paper (excluding \cref{sec:stochastic}), we assume losses are also twice differentiable and smooth, i.e., that $\nabla^2 f(w) \preceq \beta I$ for some parameter $\beta > 0$ and all $w \in W$.
The game proceeds for $T$ rounds, where in round $t$ the player chooses $w_t \in W$, suffers loss $f_t(w_t)$, and observes $f_t$ as feedback. We denote by $\Re_T$ the player's regret;
\[
    \Re_T := \sum_{t=1}^T f_t(w_t) - \min_{w\in W} \sum_{t=1}^T f_t(w),
\]
and by $\Sw_T$ the number of decision switches she performs;
$
    \Sw_T := \sum_{t=1}^{T-1} \Ind{w_{t+1} \neq w_t }.
$
When it is not clear from context, we may write $\Re_T(\Alg) $ and $ \Sw_T(\Alg)$ to make explicit which player we are referring to.
We are interested in the asymptotic behavior of the player's regret, under the restriction she is obligated to perform a limited number $S \in [T]$ of switches in expectation;
$ \E \Sw_T \leq S$.
We say $\Alg$ is an \Ssoco algorithm if it satisfies that for any loss sequence $ \E \Sw_T \leq S$.

\subsection{Basic definitions and tools}

We operate over $\R^d$,
denote the $p$-norm by  $\norm{\cdot}_p$, and omit the subscript for the Euclidean norm; meaning $\norm{\cdot} \eqq \norm{\cdot}_2$.
The \emph{diameter} of a set $\Dom \subseteq \R^d$ is defined as $\max_{x,y\in\Dom} \norm{x-y}$.
We denote by $\Projw{x} := \argmin_{w\in \Dom} \norm{w - x}^2$ the orthogonal projection of a point $x\in \R^d$ onto $\Dom$, but usually omit the subscript $\Dom$ and write $\Proj{x}$ unless the context requires to be explicit.
For two probability measures $\cQ ,\cQ'$ over a sample space $\mathcal X$, 
we write $\norm{\cQ - \cQ'}_{TV}$ to denote their \emph{total variation distance};
\[
    \TV{\cQ - \cQ'} := \max_{B \subseteq \mathcal X} \big \{ \cQ(B) - \cQ'(B) \big \}.
\]
Also, a two dimensional random variable $(\rv X, \rv Y)$ is a \emph{coupling} of $p$ and $q$ if its marginals satisfy $\rv X \sim p$ and $ \rv Y \sim q$.
Given a scale parameter $\sigma > 0$, we denote by $\Lap(\sigma)$ a multivariate Laplace distribution, with density $\nu$ given by;
\begin{align}\label{eq:def_laplace}
	\nu(z) = \frac{1}{(2\sigma)^d} \exp\br{- \norm{z}_1 / \sigma}.
\end{align}

\paragraph{Truncated domain and barrier functions.}
Given a convex and compact domain $W \subset \R^d$ and parameter $\gamma > 0$, we let 
\begin{align}
	W_\gamma \eqq \cb{w_0 + (1-\gamma)(w-w_0) : w \in W}
\end{align}
denote the $\gamma$ truncated domain, where $w_0$ may be chosen arbitrarily from $W$.
Note that for any $w \in W$, $w - (w_0 + (1-\gamma)(w-w_0)) = \gamma (w - w_0)$. Thus if $D$ is the diameter of $W$, for any $w \in W$ there exists $w_\gamma \in W_\gamma$ such that $\norm{w - w_\gamma} \leq \gamma D$.

We say $B\colon \R^d \to \R \cup \infty$ is a \emph{barrier function} (also known as Legendre; see \citealp{cesa2006prediction}) on $W$ if $B$ is non-negative, continuous, convex, and $B(w) \in \R$ for $w \in {\rm int}(W)$, and $B(w) = \infty$ otherwise.
We assume we have access to a family of barrier functions $B_{\gamma, \alpha}$ that satisfy $B_{\gamma, \alpha}(w) \leq \alpha$ for all $w \in W_\gamma$. This can be satisfied for any $\alpha, \gamma$ by scaling a given barrier function with an appropriate factor.

\subsection{Sampling from maximal couplings}
\label{sec:max_couplings}

\cref{alg:lazysample} presented below provides a mechanism to maximally couple consecutive decision distributions.  A similar procedure for sampling from maximal couplings can be found in the literature in various places, see e.g., \cite{jacob2020unbiased}. 
The desired properties of the algorithm follow from the two lemmas stated next. For completeness, we provide their proofs in \cref{rej_appendix}.
Throughout the paper, within an algorithmic context, we use the calligraphic font ($\dist P, \dist Q$, etc.) to denote computational objects that provide $O(1)$ oracle access to evaluate the density and to sample from a probability distribution.

\begin{algorithm}[ht]
    \caption{LazySample}
        \label{alg:lazysample}
    \begin{algorithmic}[1]
        \STATE \textbf{input:} $x$, $\dist Q$, $\dist P$
        \STATE Sample $z \sim \mathrm{Unif}[0, \dist Q (x)]$ 
        \STATE If $\dist P(x) > z$, \textbf{return} $x$ 
        \STATE Otherwise, repeat;
                \STATE \quad Sample $y \sim \dist P$, and $z' \sim \mathrm{Unif}[0, \dist P (y)]$
                \STATE \quad If $z' > \dist Q(y)$, \textbf{return} $y$
    \end{algorithmic}
\end{algorithm}

\begin{lemma}\label[lem]{lem:lazysample-switch}
    Running \text{LazySample}($x, \dist Q, \dist P$) with $x \sim \dist Q$, we have that $\dist P$ is sampled from with probability $\norm{\dist Q - \dist P}_{TV}$, where randomness is over choice of $x$ and execution of the algorithm.
\end{lemma}

\begin{lemma}\label[lem]{lem:lazysample-dist}
    Assume we run \text{LazySample}($x, \dist Q, \dist P$) with $x \sim \dist Q$, and that $\dist P, \dist Q$ are density functions that can be evaluated at any point and sampled from in polynomial time.
    Then the algorithm generates a return value distributed according to $\mathcal P$ in expected polynomial time.
\end{lemma}

\section{Lazy OCO}
\label{sec:ub}
In this section, we present and analyze our lazy OCO algorithm for convex losses in the oblivious adversarial setup. The algorithm has a regularization component and generates decisions that are minimizers of a perturbed cumulative loss on each round. As such, it can be viewed as a natural combination of the well known Follow-the-Perturbed-Leader (FPL) algorithm \citep{kalai2005efficient} and Follow-the-Regularized-Leader (FTRL) meta-algorithm, with regularization being intrinsic in the strongly convex case. 
The resulting algorithm, given in \cref{alg:lftprl}, is thus named Follow-the-Perturbed-Regularized-Lazy-Leader (FTPRLL).

The key idea is that stability introduced by regularization causes minimizers of the \emph{unperturbed} objectives to move in small steps, thereby encouraging consecutive decisions---minimizers of the \emph{perturbed} objectives---to overlap in total variation. This, combined with the lazy sampling sub-routine \cref{alg:lazysample}, produces a low switching algorithm.
Importantly, we note that while in FPL the perturbations are the source of stability, here regularization accounts for stability, and the perturbations serve to obfuscate the shifts between consecutive decisions.
Next, we present notation associated with our algorithm and provide its pseudocode subsequently.

\paragraph{Notation.}
Given a regularizer $R\colon W \to \R$, a barrier function $B\colon \R^d \to \R \cup \infty$, and a sequence of loss functions $f_1, \ldots, f_{t-1}$, we define: 
\begin{alignat}{3}
	&\phi_t(w; p) &~\eqq~& \sum\nolimits_{i=1}^{t-1} f_i(w) + p\T w + R(w) + B(w),
	\label{eq:objective_p_t}
	\\
	&w_t(p) &~\eqq~& \argmin_{w \in W} \phi_t(w; p).
\end{alignat}
Given a random perturbation vector  $p_t \sim \nu_t$, we let $\mathcal Q_t$ denote the density function of the random variable $w_t(p_t)$:
\begin{align*}
	\mathcal Q_t(w) = \Pr_{p_t \sim \nu_t }(w_t(p_t) = w).
\end{align*}
With slight notation overloading, we let $\cQ_t$ also denote the distribution of $w_t(p_t)$, and write $w_t(p_t) \sim \cQ_t$.

\begin{algorithm}[ht]
    \caption{Follow-The-Perturbed-Regularized-Lazy-Leader (FTPRLL)}
	\label{alg:lftprl}
    \label{alg:v1}
	\begin{algorithmic}[1]
	    \STATE \textbf{input:} 
	        $\sigma_1,\ldots,\sigma_T \in \R$, regularizer $R$, barrier $B$
        \STATE Sample $p_1 \sim \Lap(\sigma_1)$
        \STATE $ 
                w_1 \leftarrow w_1(p_1)
            $
        \FOR{$t=1$ to $T$}
            \STATE Play $w_t$, Observe $f_t$
            	\STATE Let $\cQ_{t+1}$ denote the density of $w_{t+1}(p_{t+1})$, where $p_{t+1} \sim \Lap(\sigma_{t+1})$
            \STATE $w_{t+1} \leftarrow \mathrm{LazySample}
            		\big(w_t, 
                \cQ_t, \cQ_{t+1})$
        \ENDFOR
	\end{algorithmic}
\end{algorithm}

We note that given second order access to the loss functions, barrier and regularizer,
\cref{alg:lftprl} is polynomial-time efficient. Indeed, the density functions $\cQ_t, \cQ_{t+1}$ 
passed as arguments into \cref{alg:lazysample}
may be sampled from and evaluated efficiently at any given point, as required by 
\cref{lem:lazysample-dist}. To generate a sample from $\dist Q_t$, we sample $p \sim \nu_t$ and compute $w_t(p)$ by minimizing $\phi_t(\cdot;p)$ over $W$. 
In order to evaluate $\dist Q_t$ at a given $w\in W$,
we follow the same approach presented in \cite{agarwal2023differentially}, and use a closed form expression based on the change of variables formula (see \cref{lem:cov} in \cref{sec:lem:tv_decisions} for the details):
\begin{align*}
    \cQ_t(w) = \nu_t(-\nabla \phi_t(w;0) )
        \av{\det \br{ -\nabla^2 \phi_t(w;0)}}.
\end{align*}

In what follows, we provide a regret analysis of \cref{alg:lftprl} in the general convex and strongly convex cases.
We begin with the following lemma which establishes a bound on the total variation between consecutive decision distributions of the algorithm. The lemma follows from arguments given in \citet{agarwal2023differentially}; for completeness, we provide a proof in \cref{sec:lem:tv_decisions}.

\begin{lemma}
\label{lem:tv_decisions}
    Assume the loss sequence $f_1, \ldots, f_T\colon W \to \R$ is convex, $G$-Lipschitz, and $\beta$-smooth, and that $\phi_t$ is $(1/\eta_t)$-strongly convex.
    Then, for any $t$ such that $\sigma_{t+1} = \sigma_t$, we have
    \begin{align*}
    	\TV{\cQ_{t+1} - \cQ_{t}}
		\leq \eta_t \beta d + \sqrt d G/\sigma_t.
    \end{align*}
\end{lemma}
Combined with \cref{lem:lazysample-switch}, this ensures the switch probability in any single round is bounded by the quantity above that can be controlled with the choice of step size and perturbation scale parameters.

\subsection{The general convex case}
\label{sec:ftprl_main}

We start with the general convex case where the regret analysis is simpler. Here we introduce stability into the algorithm by means of L2 regularization, with $R(w) = \frac{1}{2\eta}\norm{w - w_0}^2$ for some $w_0 \in \Dom$. Below, we state and prove our theorem giving the guarantees of \cref{alg:lftprl} when tuned for general convex losses. 
\begin{theorem} \label{thm:lftprl}
	Let $W \subset \R^d$ be convex and of diameter $D$, and assume
    the loss functions $f_1,\ldots, f_T$ are $G$-Lipschitz, $\beta$-smooth and convex over $W$.
    Further, assume the barrier $B$ satisfies $0 \leq B(w) \leq 1$ for all $w \in W_\gamma$, where $\gamma\eqq 1/T$.
    Then running \cref{alg:lftprl} with $\sigma_t = \sigma$ for all $t$ and $R(w) = \frac{1}{2 \eta}\norm{w - w_0}^2$, $w_0\in \Dom$, we obtain
    \begin{align*}
        \E \Re_T \leq 
            2\eta G^2 T 
            + \frac{D^2}{2\eta} + \sigma \sqrt d D + GD + 1
        \quad \text{ and } \quad 
        \ESw \leq \eta \beta d T + \sqrt d G T/\sigma.
    \end{align*}
    In particular, setting $\sigma = 2 \sqrt d GT/S$ and 
    $\eta = \min\cb{D/2 G \sqrt T, S/\beta d T}$ we obtain $\ESw \leq S$ and 
    $\E \Re_T = O(\sqrt T + d T/ S)$.
\end{theorem}
\begin{proof}
    Define the random minimization objective at time $t$ by;
	\begin{align}
		\tilde \phi_t(w) &\eqq 
		\phi_t(w; p_t)
		=\sum_{i=1}^{t-1} f_i(w) + p_t\T w + R(w) + B(w),
	\end{align}
    and set
    \begin{equation} 
        \label{eq:lftpl_y}
        y_{t+1} := \argmin_{w\in W} \left\{ 
                        \tilde \phi_t(w) + f_t(w)
                    \right\}.
    \end{equation}
    We have,
    \begin{align}
    \E \Re_T 
    &= \E \Big[ \sum_{t=1}^T f_t(w_t) - f_t(w^*) \Big] \nonumber
    \\
    &= \E \Big[ \sum_{t=1}^T f_t(w_t) - f_t(y_{t+1}) \Big]
        + \E \Big[ \sum_{t=1}^T f_t(y_{t+1}) - f_t(w^*) \Big], \label{eq:lftpl_1}
    \end{align}
    where the expectation is taken over randomness of the algorithm originating from random the perturbations $p_1, \ldots, p_T$.
    To bound the first term in \cref{eq:lftpl_1}, consider any perturbation $p_t$, and note that 
    \[
        \tilde \phi_t(w) + f_t(w) = \sum_{i=1}^t f_i(w) + p_t\T w + \frac{1}{2\eta}\norm{w - w_0}^2 + B(w),
    \]
    thus $\tilde \phi_t + f_t$ is $1/\eta$-strongly-convex. In addition, $\tilde \phi_t + f_t$ is minimized over $\Dom$ by $y_{t+1}$, and $\tilde \phi_t$ is minimized by $w_t$ over $\Dom$. Therefore by a standard bound on the stability of minimizers of strongly convex objectives (see \cref{lem:sc_stab} in \cref{sec:lem_tech}) we obtain
    $
        \norm{y_{t+1} - w_t} \leq 2 \eta G,
    $
    and then
    \begin{equation*} 
        \Exp{\sum_{t=1}^T f_t(w_t) - f_t(y_{t+1}) }
        \leq \Exp{ \sum_{t=1}^T G\norm{w_t - y_{t+1}} }
        \leq 2 \eta G^2 T.
    \end{equation*}
    To bound the second (leaders regret) term of \cref{eq:lftpl_1}, we first obtain a bound the w.r.t. the truncated domain $W_\gamma$.
    \begin{lemma} \label[lem]{lem:lftpl_btl}
        The hypothetical leaders regret is bounded as,
        \[
            \E \Big[ \sum_{t=1}^T f_t(y_{t+1}) - f_t(w^*_\gamma) \Big] 
            \leq \frac{D^2}{2\eta} + \sigma \sqrt d D + 1,
        \]
        for any $w^*_\gamma \in W_\gamma$.
    \end{lemma}
    Now, note that for any $w^* \in W$, there exists $w^*_\gamma = w_0 + (1-\gamma)(w^*-w_0) \in W_\gamma$ such that $\norm{w^* - w^*_\gamma} \leq \gamma D$.
    Thus, 
    \begin{align*}
            \E \Big[ \sum_{t=1}^T f_t(y_{t+1}) - f_t(w^*) \Big] 
            \leq \frac{D^2}{2\eta} + \sigma \sqrt d D + GDT\gamma + 1
            = \frac{D^2}{2\eta} + \sigma \sqrt d D + GD + 1,
    \end{align*}
    which concludes the proof of our regret bound. 
    For the switches bound, by \cref{lem:tv_decisions};
    \begin{align*}
		\ESw = \sum_{t=2}^{T}\Pr(w_{t+1} \neq w_t)
		= \sum_{t=2}^{T} \TV{\cQ_{t+1} - \cQ_t}
		\leq
		\eta \beta d T + \sqrt d G T/\sigma,
	\end{align*}
	as claimed.
\end{proof}

The proof of \cref{lem:lftpl_btl} is a straightforward adaptation of the analysis for the linear case laid out in \cite{kalai2005efficient}, and makes use of the well known Follow-the-leader Be-the-leader Lemma, stated next for completeness. 
\begin{lemma} [FTL-BTL, \citealp{kalai2005efficient}] \label[lem]{lem:btl} 
    Let $\mathcal X\subset \R^d$ be convex and compact, and  $h_1, \ldots, h_T \colon \mathcal X \to \R$ be a sequence of losses. Then, for 
    $ w_t^* := \argmin_{w \in \mathcal X} \sum_{s=1}^t h_t(w)$, we have
    \[
        \sum_{t=1}^T h_t(w_t^*)
            \leq 
        \sum_{t=1}^T h_t(w_T^*).
    \]
\end{lemma} 

\begin{proof}[of \cref{lem:lftpl_btl}]
    Fix a perturbation sequence $p_1, \ldots, p_T$, and additionally define $p_0 = 0$.
    Consider the auxiliary loss sequence 
    $\tilde f_0(w) = R(w) + B(w)$, 
    and for $t \geq 1$, $
        \tilde f_t(w) := f_t(w) + (p_t - p_{t-1})\T w
    $.
    From \cref{eq:lftpl_y} it follows that
    \[
        y_{t+1} = \argmin_{w\in W} \left\{ \sum_{i=0}^t \tilde f_i(w) \right\},
    \]
    hence the BTL Lemma (\cref{lem:btl}) we obtain (for any $w^*_\gamma \in W_\gamma$);
    $
        \sum_{t=0}^T \tilde f_t(y_{t+1}) 
            \leq \sum_{t=0}^T \tilde f_t( w^*_\gamma)
    $.
    Substituting for the definition of $\tilde f_t$ and rearranging we get
    \begin{align*}
        \label{eq:lftpl_sc_btl}
        \sum_{t=1}^T f_t(y_{t+1}) - f_t(w^*_\gamma) 
        &\leq R(w^*_\gamma) - R(y_1) + B(w^*_\gamma) - B(y_1)
        + \sum_{t=1}^T (p_t - p_{t-1})\T(w^*_\gamma - y_{t+1})
        \\
        &\leq \frac{D^2}{2\eta} + 1
        + \sum_{t=1}^T (p_t - p_{t-1})\T(w^*_\gamma - y_{t+1}).
    \end{align*}
    Now, consider any perturbations distribution $\dist Q$, such that the marginals of the $p_t$'s under $\dist Q$ are the same as the marginals of the $p_t$'s under our actual lazy algorithm which we denote by $\Alg$. Recall that $y_{t+1}$ defined by \cref{eq:lftpl_y} depends only on randomness introduced by $p_t$. This implies that the $y_t$'s are distributed the same under both $\Alg$ and $\dist Q$ as long as the marginals of the perturbations match. 
    Hence
    \begin{align*}
        \Expp{\Alg} {\sum_{t=1}^T f_t(y_{t+1}) - f_t(w^*_\gamma)} 
            &= \sum_{t=1}^T  \Expp{\Alg} {f_t(y_{t+1}) - f_t(w^*_\gamma)}
            \\
            &= \sum_{t=1}^T  \Expp{\dist Q} {f_t(y_{t+1}) - f_t(w^*_\gamma)}
            = \Expp{\dist Q} {\sum_{t=1}^T f_t(y_{t+1}) - f_t(w^*_\gamma)}.
    \end{align*}
    By \cref{lem:lazysample-dist}, for all $t$ it holds that $p_t \sim \Lap(\sigma)$ when generated by our algorithm $\Alg$. Therefore, choosing $\dist Q$ by letting $p_1 \sim \Lap(\sigma)$, and setting $p_t = p_1$ for all $t \geq 2$, we achieve the same marginals as those induced by $\Alg$. This implies
    \begin{align*}
        \Expp{\Alg} {\sum_{t=1}^T f_t(y_{t+1}) - f_t(w^*_\gamma)} 
            &\leq \frac{D^2}{2\eta} + 1
            + \Expp{\dist Q} {\sum_{t=1}^T (p_t - p_{t-1})\T(w^*_\gamma - y_{t+1})}
            \\
            &\leq \frac{D^2}{2\eta} + 1
            + \Expp{\dist Q} {\sum_{t=1}^T \norm{p_t - p_{t-1}} 
                \norm{w^*_\gamma - y_{t+1}}}
            \\
            &\leq \frac{D^2}{2\eta} + 1
            + D \Expp{\dist Q} {\sum_{t=1}^T \norm{p_t - p_{t-1}}}
            \\
            &= \frac{D^2}{2\eta} + 1 + \sigma \sqrt d D,
    \end{align*}
    as desired.
\end{proof}

\subsection{The strongly convex case}
\label{sec:scub:main}

In this section, we state and prove \cref{thm:lftprl_sc} providing the guarantees of \cref{alg:lftprl} for the strongly convex setting. The performance here hinges on increasing the perturbations variance at a certain rate, accounting for the increasing curvature in the per round minimized objective.
This, along with a careful analysis of the perturbed leaders regret, is key to achieving the quadratic gain in the guarantee.

\begin{theorem} \label{thm:lftprl_sc}
Let $W \subset \R^d$ be convex and of diameter $D$, and assume
    the loss functions $f_1,\ldots, f_T$ are $G$-Lipschitz, $\beta$-smooth and $\lambda$-strongly convex over $W$.
    Further, assume the barrier $B$ satisfies $0 \leq B(w) \leq 1$ for all $w \in W_\gamma$, where $\gamma\eqq 1/T$,
    and let $\sigma_t = \sqrt t \sigma$ for all $t\in \tau \eqq \{2, 2^1, \ldots, 2^{\lceil \log T \rceil}\}$, and $\sigma_t = \sigma_{t-1}$ for $t\notin \tau$.
   Then running \cref{alg:lftprl} with parameters $\cb{\sigma_t}$, $R(w) \equiv 0$ and barrier $B$, we obtain:
    \[
        \E \Re_T \leq \frac{4 G^2 + 2 d \sigma^2}{\lambda} (1 + \log T) + 1
        \quad \text{ and } \quad 
        \ESw \leq \frac{2\beta d }{\lambda } (1 + \log T)
		+ \frac{G \sqrt { d T } \log T}{\sigma}.
    \]
    In particular, for any $S \geq \frac{8 \beta d }{\lambda } \log T$, setting $\sigma = 6 G \sqrt {d T}/S$ we obtain the switching guarantee $\ESw \leq S $ and regret
    $\E \Re_T = O (d^2 T \log T / S^2)$.
\end{theorem}
\begin{proof}
As in the convex case, let $y_{t+1} = \argmin_{w\in\Dom} \left\{ \tilde  \phi_t(w) + f_t(w) \right\}$, where
	\begin{align}
		\tilde \phi_t(w) &\eqq 
		\phi_t(w; p_t)
		=\sum_{i=1}^{t-1} f_i(w) + p_t\T w + B(w),
	\end{align}
and by following an argument similar to the proof of \cref{thm:lftprl}, we obtain
\begin{align}
    \E \Re_T 
    &\leq \frac{2 G^2}{\lambda}(1 + \log T)
        + \E \Big[ \sum_{t=1}^T f_t(y_{t+1}) - f_t(w^*) \Big],
\end{align}
with the only difference being that $\tilde \phi_t + f_t$ is now $t\lambda$-strongly-convex.
To bound the second term in the above display, we follow the same argument given in the proof of \cref{lem:lftpl_btl} to obtain for any $w^*_\gamma \in W_\gamma$;
\begin{align*}
    \E \sbr { \sum_{t=1}^T f_t(y_{t+1}) - f_t(w^*_\gamma ) }
    \leq 1 + \E_{\dist Q} \sbr{\sum_{t=1}^T (p_t - p_{t-1})\T(w^*_\gamma  - y_{t+1})},
\end{align*}
where this time we define $\dist Q$ by
$\xi \sim \Lap(1)$, and set $p_t = \sigma_t \xi$ for all $t \in [T]$, so that $p_t \sim \Lap(\sigma_t)$ under $\cQ$.
Indeed, by \cref{lem:lazysample-dist} it holds that under our actual algorithm $p_t \sim \Lap(\sigma_t)$ as well, and therefore the marginals match those induced by $\cQ$.
Next, we exploit the fact that $p_t$ are zero mean in order to get rid of the non-random part of $w^*_\gamma  - y_{t+1}$. To that end, set 
$
    x_{t+1} := \argmin_{w\in \Dom} \left\{ 
        \sum_{i=1}^t f_i(w)
    \right\}
$,
and note $x_{t+1}$ is deterministic. Therefore,
\begin{align*}
    \E \sbr {\sum_{t=1}^T f_t(y_{t+1}) - f_t(w^*_\gamma)} - 1
        &\leq \E_{\dist Q} \sbr{\sum_{t=1}^T (p_t - p_{t-1})\T(w^*_\gamma  - y_{t+1})}
        \\
        &= \sum_{t=1}^T \E_{\dist Q} \sbr{p_t - p_{t-1}}\T(w^*_\gamma  - x_{t+1})
            + \E_{\dist Q} \sbr{\sum_{t=1}^T (p_t - p_{t-1})\T(x_{t+1} - y_{t+1})}
        \\
        &= \E_{\dist Q} \sbr{\sum_{t=1}^T (p_t - p_{t-1})\T(x_{t+1} - y_{t+1})}.
\end{align*}
    Now, note that $x_{t+1}$ and $y_{t+1}$ minimize the same $t\lambda$-strongly-convex objective up to the additional perturbation vector $p_t$, therefore by \cref{lem:sc_stab}; $\norm{x_{t+1} - y_{t+1}} \leq \frac{2 \norm{p_t}}{t \lambda}$.
    Thus we obtain,
	\begin{align*}
        \E \sbr{\sum_{t=1}^T f_t(y_{t+1}) - f_t(w^*)} - 1
            &\leq \E_{\dist Q} \sbr{\sum_{t=1}^T (p_t - p_{t-1})\T(x_{t+1} - y_{t+1})}
            \\
            &\leq \E_{\dist Q} \sbr{\sum_{t=1}^T \norm{p_t - p_{t-1}} \norm{x_{t+1} - y_{t+1}}}
            \\
            &\leq \E_{\dist Q} \sbr {\sum_{t=1}^T \norm{p_t - p_{t-1}} \frac{2 \norm{p_t}}{t \lambda} }
            \\
            &= \E_{\xi \sim \Lap (1)} \sbr{\sum_{t=1}^T \norm{ \sigma_t \xi -  \sigma_{t-1} \xi} 
                 \frac{2 \norm{ \sigma_t \xi}}{t \lambda} }
            \\
            &= \E_{\xi \sim \Lap (1)}[\norm{\xi}^2] \sbr{
            		\sum_{t=1}^T \av{ \sigma_t -  \sigma_{t-1}} 
                 \frac{2 \av{ \sigma_t }}{t \lambda} }
            \\
            &= \frac{4 d}{\lambda} 
                \sum_{t=1}^T \frac{\sigma_t(\sigma_t - \sigma_{t-1})}{t}
            \tag{$\E[\norm{\xi}^2] = 2 d$}
            \\
            &= \frac{4 d}{\lambda} 
                \sum_{t\in \tau} \frac{\sigma_t(\sigma_t - \sigma_{t-1})}{t}
            \tag{$t \notin \tau \implies \sigma_t = \sigma_{t-1}$}
            \\
            &= \frac{4 d \sigma^2 }{\lambda} 
                \sum_{t\in \tau} \frac{\sqrt t(\sqrt t - \sqrt {t/2})}{t}
            \\
            &\leq \frac{2 d \sigma^2 }{\lambda} 
                (1 + \log T).
    \end{align*}
    The above, after accounting for the truncation, gives;
    \begin{align*}
    	\E \sbr { \sum_{t=1}^T f_t(y_{t+1}) - f_t(w^*) }
    	\leq \frac{2 d \sigma^2 }{\lambda} 
                \br{1 + \log T} + GD + 1
        \leq \frac{2 d \sigma^2 }{\lambda} 
                \br{1 + \log T} + 2G^2/\lambda + 1.
    \end{align*}
    which concludes the proof of the regret bound. 
    For the switches guarantee, 
    note that since losses are $\lambda$-strongly convex, we have that $\phi_t$ are $(1/\eta_t)$-strongly convex with $\eta_t = 1/(\lambda t)$. 
    Thus, by \cref{lem:tv_decisions}; 
    \begin{align*}
		\ESw = \sum_{t=2}^{T}\Pr(w_{t+1} \neq w_t)
		= \sum_{t=2}^{T} \TV{\cQ_{t+1} - \cQ_t}
		&\leq
		1 + \log T + \sum_{t\in \tau} t( \eta_t \beta d + \sqrt d G/\sigma_t )
		\\
		&=
		1 + \log T + \sum_{t\in\tau} \frac{\beta t d }{\lambda t }  + \sqrt d G \sqrt t/\sigma 
		\\
		&\leq
		\frac{2\beta d }{\lambda } (1+\log T)
		+ \frac{G \sqrt { d T } }{\sigma(\sqrt 2 - 1)},
	\end{align*}
	where the first inequality follows since our epoch schedule contains $\lceil \log T \rceil$ perturbation scale changes; $\sqrt {t/2} \sigma = \sigma_{t-1} \to \sigma_t = \sqrt {t} \sigma$ for $t\in \tau$, and the last since it is a geometric sum starting at $\sqrt 2$.
	Finally, note that for any 
	$S \geq \frac{8 \beta d }{\lambda } \log T$,
	setting $\sigma = 6 G \sqrt {d T} /S$ gives
	\begin{align*}
		\ESw \leq \frac{2 \beta d }{\lambda } (1+ \log T) + \frac{S}{2}
		\leq \frac{4 \beta d }{\lambda } \log T
		+ \frac{S}{2}
		\leq S,
	\end{align*}
	and completes the proof.
\end{proof}

\section{Lazy Stochastic OCO}
\label{sec:stochastic}
In this section, we present a simple algorithm for the special case where the losses are drawn i.i.d.~from some distribution of convex losses $\dist F$. The standard objective to be minimized here is the \textit{pseudo} regret, defined by
\[
    \overline\Re_T := \Exp{\sum_{t=1}^T f_t(w_t) - f_t(w^*)},
\]
where $w^* = \argmin_{w\in \Dom} \Exp { f_1(w) }$ and the expectation is over the loss distribution $\dist F$, from which $f_1,\ldots,f_T$ are sampled i.i.d.
Importantly, the minimizer of the expected loss defined above stays fixed for the duration of the game. This is in stark contrast to the situation of the general adversarial setting, and enables significantly better bounds achieved by non uniform blocking as outlined by \cref{alg:lazyogd}. 
\begin{algorithm}[ht]
    \caption{Lazy SGD}
    \label{alg:lazyogd}
	\begin{algorithmic}[1]
	    \STATE \textbf{input:} learning rates $\eta_1,\ldots,\eta_T > 0$
	    \STATE $k \leftarrow 0$; arbitrary $x_1 \in \Dom$
	    \FOR{$t=1$ to $T$}
	        \IF{$t = 2^k$}
	            \STATE $k \leftarrow k+1$
                \STATE $\tilde w_k \leftarrow \Proj{\frac{1}{t}\sum_{s=1}^t x_t}$
            \ENDIF
            \STATE Play $w_t = \tilde w_k$; Observe $f_t$
            \STATE $x_{t+1} \leftarrow x_t - \eta_t \nabla f_t(x_t)$
        \ENDFOR
	\end{algorithmic}
\end{algorithm}
Next, we state and prove \cref{thm:stochastic} which summarizes the guarantees of \cref{alg:lazyogd}.

\begin{theorem}
    \label{thm:stochastic}
    Assume $\dist F$ is a distribution of $G$-Lipschitz convex losses over a domain $\Dom$ of diameter $D$.
    Then running \cref{alg:lazyogd} with step size $\eta_t = D / G \sqrt t$ guarantees $\Sw_T \leq 1 + \log T$ and
    \[
        \overline\Re_T \leq 6 D G \sqrt T  .
    \]
    If we further assume losses sampled from $\dist F$ are $\lambda$-strongly-convex, then running \cref{alg:lazyogd} with step size $\eta_t = 1/\lambda t$ guarantees $\Sw_T \leq 1 + \log T$ and
    \[
        \overline\Re_T 
        \leq \frac{2 G^2}{\lambda} \log^2 T.
    \]
\end{theorem}
\begin{proof}
    For the general convex case, observe that the iterates $x_t$ maintained by the algorithm are just decision variables of standard OGD with decreasing step $\eta_t = D / G \sqrt t$. By well known arguments (see e.g., \citealp{hazan2019introduction}) these obtain an any time $t\in [T]$ guarantee of
    \[
        \sum_{s=1}^t f_s(x_s) - f_s(w) \leq 2 D G \sqrt t,
    \]
    for any $w\in \Dom$.
    Therefore, we have for any $t$,
    \begin{align*}
        \Exp {f_t(\tilde w_k) - f_t(w^*)} 
        \leq \frac{1}{2^k} \sum_{s=1}^{2^k} \Exp { f_t(x_s) - f_t(w^*) }
        = \frac{1}{2^k} \Exp { \sum_{s=1}^{2^k}  f_s(x_s) - f_s(w^*) }
        \leq \frac{2 D G \sqrt{2^k}}{2^k} .
    \end{align*}
    Now, set $T_k := \min\{ 2^k, T+1 \}$ and we obtain
    \begin{align*}
        \Exp{\sum_{t=1}^T f_t(w_t) - f_t(w^*)}
        &=  \sum_{k = 0}^{\lfloor \log T \rfloor} \sum_{t=T_k}^{T_{k+1}-1} \Exp{ f_t(\tilde w_k) - f_t(w^*) }
        \\
        &\leq   \sum_{k = 0}^{\lfloor \log T \rfloor} \sum_{t=T_k}^{T_{k+1}-1} \frac{2 D G \sqrt{2^k}}{2^k} 
        \\
        &\leq 2 D G  \sum_{k = 0}^{\lfloor \log T \rfloor} \sqrt{2^k}
        \\
        &\leq 2 D G \frac{2^{\lfloor \log T \rfloor/2}}{\sqrt 2 - 1}
        \\
        &\leq 6 D G \sqrt T
        ,
    \end{align*}
    which concludes the proof for the general convex case.
    For the strongly convex case, we similarly have;
    \begin{align*}
        \Exp {f_t(\tilde w_k) - f_t(w^*)} 
        \leq \frac{G^2}{\lambda 2^k}(1 + \log T_k),
    \end{align*}
    therefore,
    \begin{align*}
        \Exp{\sum_{t=1}^T f_t(w_t) - f_t(w^*)}
        &\leq   \sum_{k = 0}^{\lfloor \log T \rfloor} \sum_{t=T_k}^{T_{k+1}-1} \frac{G^2}{\lambda 2^k}(1 + \log T_k)
        \\
        &\leq 
        \frac{G^2}{\lambda} \log T
        +
        \frac{G^2}{\lambda}  
        \sum_{k = 0}^{\lfloor \log T \rfloor} \log (2^k)
        \\
        &\leq 
        \frac{2 G^2}{\lambda} \log^2 T
        ,
    \end{align*}
    which completes the proof.
\end{proof}

\section{Lower Bounds}
\label{sec:lb_statements}
In this section, we breifly discuss our lower bounds.
All proofs are deffered to \cref{sec:lb_proofs}.
In the general convex setting, a lower bound can be derived somewhat indirectly by previous results of \cite{geulen2010regret}, who prove a lower bound for online buffering problems in the experts setting. 
Here we provide a dedicated statement and proof for the lazy OCO setup.
\begin{theorem} 
    \label{thm:lb_convex}
    For any $S \in \mathbb N$, there exists a stochastic sequence of $1$-Lipschitz linear losses over $\Dom=[-1, 1]$, such that the expected regret of any \Ssoco algorithm is $\Omega( T / S)$.
\end{theorem}

In the strongly convex setting we prove two lower bounds, for both the adaptive and oblivious settings.
The result for the adpative case stated next, establishes the blocking technique to be optimal in this setting; 
for details, see \cref{sec:lb:sc_adaptive}.
\begin{theorem}
    \label{thm:sclb_adaptive}
    For any \Ssoco player $\Alg$ there exists a sequence of $1$-strongly-convex, $1$-Lipschitz, $1$-smooth losses $\{f_t\}$ over the domain $\Dom = [-1, 1]$ such that 
    $
        \Re_T = \Omega \left( T/S \right)
    $.
\end{theorem}

For the strongly convex oblivious setting, we prove a lower bound for a certain restricted class of players, characterized by \cref{assume:dists} below.
Loosely speaking, these are algorithms that employ continuous decision distributions with a constant portion of mass within one standard deviation from their mean.

\begin{assumption}\label{assume:dists}
    There exist constants $\cTV \geq 1 \geq \cTVv > 0$, such that  
    every pair of player decision distributions $p_1, p_2 \in \Delta[-1, 1]$, with means $\mu_1, \mu_2$ and variances $\sigma_1^2 \leq \sigma_2^2$, satisfy
    \begin{enumerate}
        \item $\TV{p_1 - p_2} \geq 
        \frac{1}{\cTV} \min\Big\{ 1, 
                \frac{|\mu_1 - \mu_2|}{\sigma_1 + \sigma_2}
        \Big\},\;\;$ or $\sigma_1 + \sigma_2 \geq \frac{1}{\cTV}$; and
        \item $\TV{p_1 - p_2} \geq 
        \frac{1}{\cTV} \min\Big\{ 1, 
        \cTVv - \frac{\sigma_1}{\sigma_2}
        \Big\}$, or $\sigma_1 \geq 1$.
    \end{enumerate}
  \end{assumption}
It is not hard to show the above assumption holds for the family of uniform distributions, and the work of \cite{devroye2018total} establishes similar properties for Gaussian distributions. In \cref{sec:assume_dists_discuss} we prove this property is satisfied by a more general class of distributions.
Subject  to \cref{assume:dists}, we are able to show our upper bound (\cref{thm:lftprl_sc}) is tight, as established by our next and final theorem.

\begin{theorem}\label{thm:lbsc_oblivious}
For any $S$-lazy OCO algorithm $\Alg$ with decision distributions that satisfy \cref{assume:dists}, there exists a sequence of $1$-strongly convex, $1$-Lipschitz, $1$-smooth losses on which the regret of $\Alg$ is $\Omega(T/S^2)$.
\end{theorem}

%% file: doc_lb_sc_full.tex
\section{Lower Bounds - Proofs}
\label{sec:lb_proofs}
In this section, we provide detailed proofs of lower bounds presented in \cref{sec:lb_statements}.

\subsection{General convex, oblivious adversary} 
\label{sec:lb_convex_proofs}
\label{sec:lb_general}

In this section, we establish
an $\Omega\left(T / S \right)$ lower bound on the expected regret of any \Ssoco algorithm in the general convex setting. We denote by $\Ber_p$ the Bernoulli distribution over $\{-1, 1\}$ that takes the value $1$ w.p. $p\in(0,1)$, and by $\Ber_p^j$ the joint distribution of $j$ independent samples from $\Ber_p$.
First consider the standard unconstrained setup in the scalar case $\Dom = [-1, 1]$. For $p,q\in(0, 1)$ sufficiently close, an adversary
that plays $f_t(w) = b_t w$, with $b_t \sim \Ber_p$ is indistinguishable from one that draws $b_t \sim \Ber_q$, and an $\Omega(\sqrt T)$ bound may be established.
In the \soco setting, the switching limit gives room for the adversary to repeat losses, effectively decreasing the amount of samples revealed and thereby allowing for a larger deviation between the loss distributions while maintaining their indistinguishability. 
The proof of \cref{thm:lb_convex} given next, provides a formal construction of this nature.

\begin{proof} [of Theorem \ref{thm:lb_convex} ]
    Fix $T\in \mathbb N$, and let $\Alg$ be an arbitrary \Ssoco algorithm.
    For any $p\in(0,1)$ we define the adversarial construction of the stochastic loss sequence $\mathcal F(p, S)$ as follows.
    Split the $T$ rounds into $J:=C^2 S^2$ sections with $\tau := \frac{T}{J}$ consecutive rounds in each, where $C\in \R$ is a universal constant that will be determined later on. At the onset of each section $j \in [J]$, draw a single sample $b_j \sim \Ber_p$ and play
    \[
        f_t(w) = b_j w \quad \forall t \in [(j-1) \tau + 1, j \tau].
    \]
    That is, the adversary commits to the loss given by the single sample $b_j$ for the entirety of section $j$. This concludes the adversarial construction $\mathcal F(p, S)$. 
    Moving forward, we consider the minimizer of the expected cumulative loss
    \[
        w^* := \argmin_{w\in W} \Expp{p}{\sum_{t=1}^T f_t(w)},
    \]
    where we use the subscript $p$ under expectation to signify $\{f_t\}_{t=1}^T$ are distributed according to $\mathcal F(p, S)$.
    Clearly, $w^*$ can perform no better than the realized minimizer in hindsight, and therefore it suffices to prove our lower bound with respect to it;
    \begin{align*}
        \E \Re_T  
            &= \Exp{\sum_{t=1}^T f_t(w_t) - \min_{w\in W}\sum_{t=1}^T f_t(w)}
            \\
            &\geq \Exp{\sum_{t=1}^T f_t(w_t) - \sum_{t=1}^T f_t(w^*)}.
    \end{align*}
    Now, set $\epsilon := \frac{1}{8 C S}$ and consider the two adversaries given by $\mathcal F(p_+, S)$ and $ \mathcal F(p_-, S)$, where
    \begin{align*}
        p_+ := \frac{1 + \epsilon}{2},
        \quad \text{and} \quad  
        p_- := \frac{1 - \epsilon}{2}.
    \end{align*}
    Respectively, denote the minimizer of the expected loss of each adversary by
    \begin{align*}
        w_+^* := \argmin_{w\in W} \Expp{p_+}{ \sum_{t=1}^T f_t(w)},
        \quad \text{and} \quad 
        w_-^* := \argmin_{w\in W} \Expp{p_-}{ \sum_{t=1}^T f_t(w)}. 
    \end{align*}
    The following lemma establishes that on every round where the player's decision and the loss are independent, the regret incurred against at least one of these adversaries will be $\Omega(\epsilon)$.
    \begin{lemma}\label[lem]{lem:lb_weak_main} 
        Denote by $w_t$ the decision of $\Alg$ on round $t \in [T]$, and assume $f_t$ and $w_t$ are independent. Then at least one of the following bounds holds;
        \begin{align*}
            \Expp{p_+}{f_t(w_t) - f_t(w_+^*)} \geq \frac{\epsilon}{4}
            \qquad \textit{or} \qquad
            \Expp{p_-}{f_t(w_t) - f_t(w_-^*)} \geq \frac{\epsilon}{4}.
        \end{align*}
    \end{lemma}    

    Importantly, since the player is allowed far fewer switches $(S)$ than there are sections $(J = C^2 S^2)$, it follows that for most sections the player's decision is indeed independent of the loss. Proceeding, we consider the decomposition of her regret into two terms; 
    \[
        \Re_T = \Re_{pos} + \Re_{neg}.
    \]
    The positive regret term $\Re_{pos}$ includes all rounds belonging to sections where the player did not switch. These are precisely the rounds on which her decision is independent of the loss. The negative regret term $\Re_{neg}$ on the other hand, includes all other rounds belonging to sections on which at least one decision switch was performed. On any such section, the loss suffered is trivially lower bounded by $-1$, which implies
    \[
        \E \Re_{neg} 
        \geq \E[-\Sw_T \cdot \tau] 
        \geq - S \frac{T}{C^2 S^2}
        = -\frac{T}{C^2 S}.
    \]
    (Recall $\Sw_T$ denotes the random variable number of switches performed by $\Alg$.)
    This leaves $J - \Sw_T$ sections that work in favor of the adversary and contribute (positively) to $\Re_{pos}$. On every round of any such section, \cref{lem:lb_weak_main} applies and the player must suffer a $\frac{\epsilon}{4}$ regret penalty from at least one of the adversaries. This implies
    \[
        \E  \Re_{pos}  
        \geq \E\Bigg[
            \frac{(J - \Sw_T) \tau \frac{\epsilon}{4}}{2}
        \Bigg]
        \geq 
            \frac{(J - S) \tau \frac{\epsilon}{4}}{2}
        = \frac{ T }{64 C S} - \frac{ T }{64 C^3 S^2}.
    \]
    Now, for a choice of $C = 128$ we obtain
    \[
        \E \Re_T   = \Exp{ \Re_{neg} + \Re_{pos} } \geq 
        \frac{T}{CS} \left(\frac{1}{64} - \frac{1}{64 C^2 S} - \frac{1}{C} \right)
        > \frac{T}{200 C S},
    \]
    which completes the proof.
\end{proof}

The proof of \cref{lem:lb_weak_main} hinges on \cref{lem:ber_tv} stated below, whose proof follows from well known information theoretic arguments and is given in \cref{sec:lem_tech} for completeness.
\begin{lemma} \label[lem]{lem:ber_tv}
    Let $\epsilon \in (0, \frac{1}{4})$, and set $p = \frac{1 + \epsilon}{2}, q = \frac{1 - \epsilon}{2}$. It holds that
    \[
        \norm{\Ber_{p}^j - \Ber_{q}^j}_{TV}
        \leq \epsilon \sqrt{2 j}.
    \]
\end{lemma}

\begin{proof}[of \cref{lem:lb_weak_main}]
    By assumption, $w_t$ and $f_t$ are independent, therefore $b_t$ is Bernoulli distributed with the appropriate parameter for both adversaries. Hence, for $p\in \{p_+, p_-\}$ we have
    \[
        \Expp{p}{f_t(w)} = (2p - 1)w,
    \]
    and
    \begin{align*}
        \Expp{p_+}{f_t(w_+^*)} = \Expp{p_-}{f_t(w_-^*)} = -\epsilon.
    \end{align*}
    Therefore
    \begin{align*}
        \Expp{p_+}{f_t(w_t) - f_t(w_+^*)} &= \epsilon\Expp{p_+}{1 + w_t}
        \\
        \Expp{p_-}{f_t(w_t) - f_t(w_-^*)} &= \epsilon\Expp{p_-}{1 - w_t}.
    \end{align*}
    Now, assume by contradiction that both bounds fail to hold. Then
    \begin{align*}
        \Expp{p_+}{1 + w_t} < \frac{1}{4}
        ; \qquad
        \Expp{p_-}{1 - w_t} < \frac{1}{4},
    \end{align*}
    and by Markov's inequality
    \begin{align*}
        \Pr_{p_+}( w_t > 0 ) < \frac{1}{4}
        ; \qquad
        \Pr_{p_-}( w_t < 0 ) < \frac{1}{4}.
    \end{align*}
    Hence, considering the event the player outputs a decision $w_t > 0$, we arrive at the conclusion the total variation between the distributions of the two adversaries is $\geq \frac{1}{2}$.
    We now apply \cref{lem:ber_tv} to arrive at contradiction;
    \begin{align*}
        \frac{1}{2}
            &\leq \Pr_{p_-} ( w_t > 0 ) - \Pr_{p_+} ( w_t > 0 )
            \\
            &\leq \norm{\Ber_{p_+}^{J} - \Ber_{p_-}^{J}}_{TV}
            \\
            &\leq \epsilon \sqrt{2 J}
            \\
            &= \frac{ \sqrt{2 C^2 S^2}}{8 C S} \leq \frac{1}{4},
    \end{align*}
    and the proof is complete.
 \end{proof}


\subsection{Strongly convex losses, adaptive adversary}
\label{sec:lb:sc_adaptive}

In this section, we prove an $\Omega(T/S)$ lower bound for an adaptive adversary with strongly convex losses.
This result establishes the blocking technique (see e.g., \citealp{arora2012online, chen2019minimax}, where it is termed mini-batching) to be optimal in this setting; It is well known (see e.g., \citealp{hazan2019introduction}) that for $G$-Lipschitz $\lambda$-strongly-convex losses the regret guarantee of OGD with decreasing step size $\eta_t=1/\lambda t$ is $O (\frac{G^2}{\lambda}\log T)$. By applying the blocking technique in this setting, we arrive at an $S$ round unconstrained game with $(T/S)G$-Lipschitz $(T/S)\lambda$-strongly-convex losses, and with a properly adjusted step size we obtain a $O((T/S)\log S)$ regret guarantee.
At a high level, we make two observations. The first is that in order to exploit strong convexity and obtain low regret, following the leader is practically mandatory in the standard setup. The next lemma formalizes this idea; The leaders loss is only within an $O(\log T)$ additive advantage over the best fixed decision in hindsight, therefore any low regret algorithm must also perform well compared to the leaders.
\begin{lemma} [Reverse BTL-lemma]
    \label[lem]{lem:rev_btl} 
    Let $f_t$ be a sequence of $1$-strongly-convex $1$-Lipschitz losses. Denote by $w_t^* := \argmin_{w \in \Dom} \sum_{s=1}^t f_t(w)$ the leader at time $t$. Then
    \[
        \sum_{t=1}^T f_t(w_T^*) - \sum_{t=1}^T f_t(w_t^*) \leq 2 \sum_{t=1}^T \frac{1}{t}.
    \]
\end{lemma}
The second observation is that on any round, an adaptive adversary may move the leader away from the player by $\Omega(1/T)$. Since the player has a limited number of switches, there will be several long sections where these small steps accumulate and result in large loss.
The proof of the theorem stated below and its proof formalize the above idea.

\begin{proof}[of \cref{thm:sclb_adaptive}]
    We consider losses of the form
    \[
        f_t(w) := \frac{1}{2}(w - x_t)^2; \quad x_t \in [-1, 1].
    \]
    As a direct implication of \cref{lem:rev_btl} we have that
    \[
        \sum_{t=1}^T f_t(w_t) - f_t(w^*) \geq \sum_{t=1}^T (f_t(w_t) - f_t(w_t^*)) - 2 \log T,
    \]
    and therefore we will be interested in lower bounding the regret with respect to the leaders. 
    On every round $t$, we choose $x_t$ such that the leader moves $1/2T$ \textbf{away} from $w_t$;
    \begin{align*}
        x_t := \begin{cases}
            w_{t-1}^* + \frac{t}{2T} \qquad w_t \leq w_{t-1}^*,
            \\
            w_{t-1}^* - \frac{t}{2T} \qquad w_t > w_{t-1}^*.
        \end{cases}
    \end{align*}
    This implies
    \[
        w_t^* = \frac{t-1}{t}w_{t-1}^* + \frac{1}{t} x_t
        = \begin{cases}
            w_{t-1}^* + \frac{1}{2T} \qquad w_t \leq w_{t-1}^*,
            \\
            w_{t-1}^* - \frac{1}{2T} \qquad w_t > w_{t-1}^*.
        \end{cases}
    \]
    Since $w_t^*$ never leaves $[-\frac{1}{2}, \frac{1}{2}]$, we can be certian $x_t\in [-1, 1]$ throughout the game and therefore the construction is valid.
    We will now limit our attention to the second half of the game and assume w.l.o.g that $S$ switches occur there;
    \begin{align}
        \Re_T + 2 \log T 
        &\geq \sum_{t=1}^T f_t(w_t) - f_t(w_t^*) \nonumber
        \\
        &\geq \sum_{t > T/2 } f_t(w_t) - f_t(w_t^*) \nonumber
        \\
        &= \sum_{j=1}^S \sum_{t=s_j}^{s_j + \tau_j} f_t(w_{s_j}) - f_t(w_t^*) \label{eq:sclb:1}
        ,
    \end{align}
    where $s_j \in \{\lceil T/2 \rceil, \ldots, T\}$ denote the player's switch rounds. To show that on each stationary section $s_j \rightarrow s_j + \tau_j$ the player's loss is $\Omega(\tau_j^2/T)$, first observe that
    \begin{align*}
        f_t(w_t) - f_t(w_t^*) 
        &= \frac{1}{2}(w_t - x_t)^2 - \frac{1}{2}(w_t^* - x_t)^2 
        \\
        &= \frac{1}{2}\left[ (w_t - w_t^*)^2 + 2(w_t - w_t^*)(w_t^* - x_t) \right]
        \\ 
        &\geq (w_t - w_t^*)(w_t^* - x_t)
        \\
        &\geq \frac{|w_t - w_t^*|}{2}. 
    \end{align*}
    In addition, whenever the player stays stationary the above distances accumulate;
    \[
        |w_k - w_{k+t}^*| \geq |w_k^* - w_{k+t}^*| = \frac{t}{2 T}.
    \]
    Hence, for all $j\in[S]$ we have
    \begin{align*}
        \sum_{t=s_j}^{s_j + \tau_j} f_t(w_{s_j}) - f_t(w_t^*) 
        \geq \sum_{t=s_j}^{s_j + \tau_j} \frac{|w_t - w_t^*|}{2}
        \geq \frac{1}{2} \sum_{t=1}^{\tau_j} \frac{t}{2 T}
        \geq \frac{\tau_j^2}{8 T}.
    \end{align*}
    Substituting the above into \cref{eq:sclb:1} we obtain
    \[
        \Re_T + 2 \log T \geq \frac{1}{8 T} \sum_{j=1}^S \tau_j^2.
    \]
    To complete the proof, we recall that $\sum_{j=1}^S \tau_j = T/2$, and invoke the following well known fact. Let $v\in \R^S$ with $\sum_{i=1}^S v_i = C$, and $u\in \R^S$ with $u_i=\frac{C}{S}$ for all $i\in[S]$, then
    \[
        \norm{v}^2 
        = \norm{v - u + u}^2
        = \norm{v - u}^2 + \norm{u}^2
        \geq S\frac{C^2}{S^2} = \frac{C^2}{S}.
    \]
    This implies
    \[
        \Re_T + 2 \log T \geq \frac{1}{8 T} \frac{T^2}{4 S} = \frac{T}{32 S},
    \]
    and we are done.
\end{proof}

\begin{proof}[of \cref{lem:rev_btl} (Reverse BTL-Lemma)]
    We prove by induction on $T$. The base case is obvious, for the inductive step observe;
    \begin{align*}
        \sum_{t=1}^T f_t(w_T^*) - \sum_{t=1}^T f_t(w_t^*)
        &= \sum_{t=1}^{T-1} f_t(w_T^*) - \sum_{t=1}^{T-1} f_t(w_t^*) \\
        &= \sum_{t=1}^{T-1} f_t(w_T^*) - f_t(w_{T-1}^*) 
            + \sum_{t=1}^{T-1} f_t(w_{T-1}^*) - f_t(w_t^*) \\
        &\leq \sum_{t=1}^{T-1} f_t(w_T^*) - f_t(w_{T-1}^*) + \sum_{t=1}^{T-1} \frac{2}{t},
    \end{align*}
    where the last inequality follows from the inductive hypothesis.
    In addition, we have
    \begin{align*}
        \sum_{t=1}^{T} f_t(w_T^*) &\leq \sum_{t=1}^{T}  f_t(w_{T-1}^*) \\ 
        \iff 
        \sum_{t=1}^{T-1} f_t(w_T^*) - f_t(w_{T-1}^*) 
        &\leq f_T(w_{T-1}^*) - f_T(w_T^*)
        \leq \frac{2}{T},
    \end{align*}
    with the last inequality follows from the fact that $\norm{w_{T-1}^* - w_T^*} \leq \frac{2}{T}$.
    Putting this together with the previous inequality we obtain 
    \begin{align*}
        \sum_{t=1}^T f_t(w_T^*) - \sum_{t=1}^T f_t(w_t^*)
        \leq \frac{2}{T} + \sum_{t=1}^{T-1} \frac{2}{t} = \sum_{t=1}^T \frac{2}{t},
    \end{align*}
    as desired. 
\end{proof}

\subsection{Strongly convex losses, oblivious adversary}
\label{sec:lb:sc_oblivious}

In this section, we prove \cref{thm:lbsc_oblivious}, establishing a regret lower bound in the strongly convex oblivious setting. Given a randomized player, we tailor an adversary to that player exploiting knowledge of her algorithm, and in particular, of her decision distributions. At a high level, we force the player to follow the leader over a long trajectory. Since the player is given a limited switching quota, the only way she can follow the leader over a long trajectory is by using high variance decisions, which imply high regret.
A link between the switching quota and the total variation between decision distributions is made formal by the next well known lemma.
\begin{lemma} [max coupling]
    \label[lem]{lem:max_coupling}
    Let $\rv X, \rv Y$ be two random variables distributed according to $p, q$ respectively over a sample space $\Omega$. Then for any joint distribution $(\rv X, \rv Y)$ we have that
    \[
        \Pr(\rv X \neq \rv Y) \geq \norm{p - q}_{TV}.
    \]
\end{lemma}
Given the above, the next lemma now links long trajectories to high variance, subject to a switching constraint.
\begin{lemma} [variance lower bound]
    \label[lem]{lem:varlb}
    Let $p, q$ be two probability measures with means $\mu_p, \mu_q$ and variances $\sigma_p^2, \sigma_q^2$. Then
    \[
        \frac{(\mu_p - \mu_q)^2}{ 2 \norm{p - q}_{TV}}
        \leq
        (\mu_p - \mu_q)^2 + \sigma_p^2 + \sigma_q^2.
    \]
    In particular, if $\norm{p - q}_{TV} \leq 1/4$, then
    \[
        (\mu_p - \mu_q)^2
        \leq
        \sigma_p^2 + \sigma_q^2.
    \]
\end{lemma}
The proofs for Lemmas \ref{lem:max_coupling} and \ref{lem:varlb} are provided in \cref{sec:lem_tech}.
Unfortunately, \cref{lem:varlb} is insufficient to prove our lower bound, and is un-improvable in general due to the existence of distributions that do not have mass near the mean, such as Bernoulli distributions.
Instead, we make use of the stronger \cref{assume:dists} on the distributions played by the player.

\begin{proof} [of \cref{thm:lbsc_oblivious}]
  Throughout the proof, we denote the player's decision distribution sequence by $\{p_t\}_{t=1}^T$, her decision expectations by $\mu_t$, and the variance of her decision by  
  $ \sigma_t^2 := \E_{w_t \sim p_t}  [(w_t - \mu_t)^2 ]$.
  The adversary construction is similar to the one in the adaptive case (\cref{thm:sclb_adaptive}).
  We consider losses of the form;
  \[
      f_t(w) := \frac{1}{2}(w - x_t)^2,
  \]
  and select $x_t$ such that the leader moves $\delta \eqq 1/2T$ \textbf{away from} $\boldsymbol \mu_t$;
    \begin{align*}
        x_t \eqq w_{t-1}^* + \delta_t \eqq \begin{cases}
            w_{t-1}^* + \frac{t}{2T} \qquad \mu_t \leq w_{t-1}^*,
            \\
            w_{t-1}^* - \frac{t}{2T} \qquad \mu_t > w_{t-1}^*.
        \end{cases}
    \end{align*}
    This implies
    \[
        w_t^* = \frac{t-1}{t}w_{t-1}^* + \frac{1}{t} x_t
        = \begin{cases}
            w_{t-1}^* + \frac{1}{2T} \qquad \mu_t \leq w_{t-1}^*,
            \\
            w_{t-1}^* - \frac{1}{2T} \qquad \mu_t > w_{t-1}^*.
        \end{cases}
    \]
    Since $w_t^*$ never leaves $[-\frac{1}{2}, \frac{1}{2}]$, we can be certian $x_t\in [-1, 1]$ throughout the game and therefore the construction is valid.
    As a first step, we derive a per round regret lower bound in terms of the player's decision expectation and variance.
    \begin{lemma}\label{lem:lb_sc_roundreg}
      For all $t\in [T]$ it holds that
      \begin{equation*}
        \E_{w_t \sim p_t} [f_t(w_t)] - f_t(w_t^*) 
        \geq
        \frac{\sigma_t^2}{2} + \frac{(\mu_t - w_t^*)^2}{2}
        + (1-t)\delta_t(\mu_t - w_{t-1}^*).
      \end{equation*}
      (We arbitrarily set $w_0^* \eqq 0$ as a matter of convenience.)
    \end{lemma}
    
    As in the adpative case, assume the game starts on round $T/2 + 1$, so that the linear term from \cref{lem:lb_sc_roundreg} is $\geq 1/4$;  $(1-t)\delta_t \geq 1/4$. (Formally, this can be accomplished for example with an adversary that plays $\delta_t = 0$ for the first $T/2$ rounds.)

  We now consider the decomposition of the regret into $J \geq 1$ sections formed by the $J-1$ rounds on which the player forces the adversary to change direction. 
  On each of these sections, the leaders progress either constantly to the right or constantly to the left. For section $j < J$ (i.e., not the last section), on the last round the leader will be bypassed by the player, causing it to change directions. 
  To ease notational clutter, we conveniently use shifted indexes when discussing a certain section; for example, $w_t^*$ and $p_t$ of section $j \in [J]$ that starts on round $t_j$, correspond to $w_{t_j + t}^*$ and $p_{t_j + t}$ of the global indexing scheme.
  For an illustration of a rightward section see \cref{fig:lb_section}.
  \begin{figure}[ht]
    \includegraphics[width=15cm]{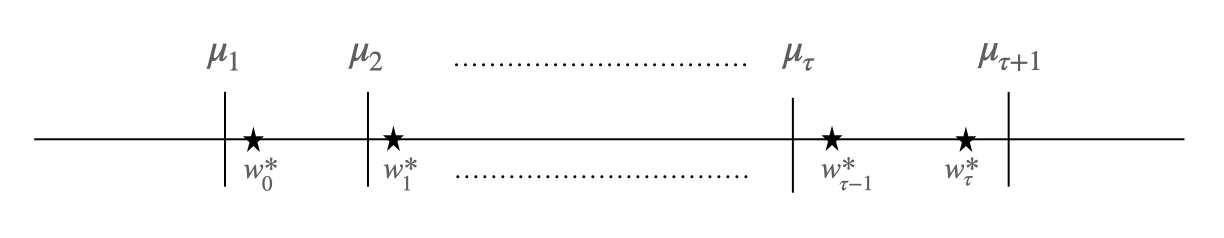}
    \centering
    \caption{A rightward path leaders section.} 
    \label{fig:lb_section}
  \end{figure}

  Our next two lemmas lower bounds the player's regret on each such individual section. Both are formulated w.l.o.g.~in terms of a rightward moving section (clearly, a leftward moving one behaves the same). The first does not make use of $\cref{assume:dists}$, and applies only in the case a generous switching budget was invested in the section. Intuitively, these are the long sections.
  \begin{lemma}
    \label{lem:lbsc_pathweak}
    Let $w_0^* \in \Dom$, 
    assume $w_t^* = w_{t-1}^* + \delta$ , 
    and $\mu_t \leq w_{t-1}^*$ for $t=1, \ldots, \tau$.
    In addition, assume the player uses a switching budget of $\zeta \geq 1/2a$ over all $\tau$ rounds for some $a \geq 1$. 
    Then it holds that;
    \begin{align*}
        \sum_{t=1}^\tau \E[f_t(w_t)] - f_t(w_t^*) 
            \geq \tau \frac{(\tau \delta)^2}{ 32^3 a^2 \zeta^2}.
    \end{align*}
\end{lemma}

Our next lemma is used to lower bound the regret on low switching sections in which the player covers the distance traveled by the leader, and crucially uses \cref{assume:dists}. 
  \begin{lemma}
    \label{lem:lbsc_main}
    Let $w_0^* \in \Dom$, 
    assume $w_t^* = w_{t-1}^* + \delta$ , 
    $\mu_t \leq w_{t-1}^*$ for $1 \leq t \leq \tau$, 
    and $\mu_{\tau+1} > w_\tau^*$ 
    .
    In addition, assume the player satisfies \cref{assume:dists} and uses a switching budget of $\zeta$ over all $\tau$ rounds;
    $\sum_{t=1}^\tau \TV{p_t - p_{t+1}} = \zeta$.
    Then, it holds that
    \[
        \sum_{t=1}^\tau \E[f_t(w_t)] - f_t(w_t^*) 
        \geq 
            \tau \frac{C_0^2}{32^3 \cTV^4}
            \min\Big\{
                1,
                \frac{(\tau \delta)^2}{\zeta^2}
            \Big\}
        .
    \]
\end{lemma}

 For the total regret minimization problem, the player faces multiple optimization problems (one per section), each of which we are able to lower bound using our above lemmas.
Letting $\Re^j(\tau_j, \zeta_j)$ denote the player's regret on the $j$'th section consisting of $\tau_j$ rounds and switching budget of $\zeta_j$, the total regret may be written as
\[
    \Re_T = \sum_{j=1}^J \Re^j(\tau_j, \zeta_j).
\]
If the last section has $\geq T/4$ rounds, our result follows by \cref{lem:lbsc_pathweak}.
Assume otherwise, then the player has invested $\geq T/4$ rounds into direction swap sections where \cref{lem:lbsc_main} applies, meaning for all $j\leq K \eqq J-1$ it holds that
\begin{align*} 
    \Re^j(\tau_j, \zeta_j) 
    \geq C \cdot \tau_j \cdot \min\Big\{
        1, \frac{(\tau_j\delta)^2}{\zeta_j^2}
    \Big\},
    \quad \textit{where } C \eqq \frac{C_0^2}{32^3 C_1^4}
    .
\end{align*} 
Hence, it follows that the player's total regret is lower bounded by the optimal solution value of the following optimization problem;

\begin{equation*}
    \begin{aligned}
        \min_{\substack{
            \tau_1, \ldots, \tau_J\\ 
            \zeta_1, \ldots, \zeta_J
            }} \quad 
            &  C \delta^2 \sum_{k=1}^K 
            \frac{\tau_k^3}{\zeta_k^2}
            \\
        \textrm{s.t.} \quad 
            & \sum_{k=1}^K \tau_k = T/4, \\
            & \sum_{k=1}^K \zeta_k = S.
    \end{aligned}
  \end{equation*}
  We have that
  \begin{align*}
    C \delta^2 \sum_{k=1}^K 
        \frac{\tau_k^3}{\zeta_k^2}
    = \frac{C T^3}{4^3 4 T^2} \sum_{k=1}^K 
        \frac{(\tau_k/(T/4))^3}{\zeta_k^2}
    = \frac{C}{4^4} \frac{T}{S^2} \sum_{k=1}^K 
        \frac{(\tau_k/T)^3}{(\zeta_k/S)^2}.  
  \end{align*}
  To conclude the proof, what remains is to show the term on the right hand side is lower bounded by a constant. Indeed, the following lemma establishes it is $\geq 1$.
  \begin{lemma}\label{lem:lbsc_kkt}
      For any $\alpha, \beta \in \Delta_K$, it holds that
      $\sum_{k=1}^K \frac{\alpha_k^3}{\beta_k^2} \geq 1$.
  \end{lemma}
  This concludes the proof.
\end{proof}

\begin{proof}[of \cref{lem:lb_sc_roundreg}]
    Observe;
    \begin{align*}
        \E f_t(w_t) - f_t(w_t^*)
        &= \frac{1}{2} \Big(
            \Exp{(w_t - x_t)^2} - (w_t^* - x_t)^2
        \Big)
        \\
        &= \frac{1}{2} \Big(
          \Exp{(w_t - \mu_t + \mu_t - x_t)^2} - (w_t^* - x_t)^2
          \Big)
        \\
        &= \frac{1}{2} \Big(
          \sigma_t^2 + 2\Exp{(w_t - \mu_t)(\mu_t - x_t)}
          + (\mu_t - x_t)^2 - (w_t^* - x_t)^2
          \Big)
        \\
        &= \frac{1}{2} \Big(
            \sigma_t^2
          + (\mu_t - x_t)^2 - (w_t^* - x_t)^2
          \Big)
        \\
        &=  \frac{\sigma_t^2}{2} + \frac{(\mu_t - w_t^*)^2}{2}
            + (\mu_t - w_t^*)(w_t^* - x_t)
            .
    \end{align*}
    In addition,
    \begin{align*}
        (\mu_t - w_t^*)(w_t^* - x_t)
        &= (\mu_t - w_{t-1}^* - \delta_t)(\delta_t + w_{t-1}^* - x_t)
        \\
        &= (\mu_t - w_{t-1}^* - \delta_t)(\delta_t -t \delta_t)
        \\
        &= (\mu_t - w_{t-1}^*)(1-t)\delta_t - (1-t)\delta_t^2
        \\
        &\geq (\mu_t - w_{t-1}^*)(1-t)\delta_t,
    \end{align*}
    and the claim immediately follows.
    \end{proof}

    \begin{proof} [of \cref{lem:lbsc_pathweak}]
        First, observe that if $|\mu_t - w_{t-1}^*| \geq (\tau\delta/16 a \zeta)^2$ on $\tau/2$ rounds, it follows by \cref{lem:lb_sc_roundreg} that
        \[ 
            \Re_\tau \geq 
                (\tau/2) \frac{ (\tau\delta)^2}{4 (16 a \zeta)^2}
                = \tau\frac{(\tau\delta)^2}{32^2 a^2 \zeta^2}
                ,
        \]
        and the result follows. Henceforth, we assume there exist $\geq \tau/2$ rounds with $|\mu_t - w_{t-1}^*| < (\tau\delta/16 a \zeta)^2$.
        Split these rounds into $8 a \zeta$ consecutive sections 
        of length $M \eqq (\tau/2)/(8 a \zeta) = \tau/16 a \zeta$ each.
        To be sure we have $\geq 2$ rounds in each section, we assume $\tau \geq 32 a \zeta$. Otherwise, for $T > 320 a$ we have
        \begin{align*}
            \tau \frac{(\tau\delta)^2}{32^2 a^2 \zeta^2}
            < \tau \delta^2 < 1/T,
        \end{align*}
        and therefore the statement holds trivially.
    
        Proceeding, by \cref{lem:max_coupling} it must hold that the sum of total variations between consecutive distributions is $\leq \zeta$,
        therefore at most $4 a \zeta$ sections may use up more than $1/4 a$ total variation (aka switching) budget.
        This establishes there are $\geq 4 a \zeta$ sections that must use up less than $1/4a$ budget. Formally, let $\{t_i\}_{i=1}^{M}$ be rounds of one of these sections, then
        \[
            \sum_{i=1}^{M-1} \TV{p_{t_{i+1}} - p_{t_i}} \leq \frac{1}{4a}.
        \]
        Since the total variation defines a distance metric, by the triangle inequality this further implies that
        $\TV{p_{t_{i + m}} - p_{t_i}} \leq 1/4a$ for all $i, i+m \in [m]$.
        Now, by \cref{lem:varlb}, for all $i \in [M/2]$ we have that
        \begin{align*}
            \sigma_{t_i}^2 + \sigma_{t_{i + M/2}}^2
            \geq
            (M\delta/2 - (\tau\delta / 16 a \zeta)^2)^2
            \geq 
            (M\delta/4)^2
            = \frac{(\tau\delta)^2}{16^2 a^2 \zeta^2}.
        \end{align*}
        To justify the second inequlity, note that 
        $1 \geq \tau\delta /16 a \zeta$, thus
        $
            (\tau\delta / 16 a \zeta)^2 
            \leq 
            \tau\delta /16 a \zeta
            =
            M\delta/4
        $.
        To conclude the proof, note we have identified a total of 
        $4 a \zeta M/2 \geq \tau/8$ round pairs that contribute
        at least $\frac{(\tau\delta)^2}{16^2 a^2 \zeta^2}$ to the regret.
        Denote these rounds $\{t_l\}_{l=1}^{\tau/8}$, then by 
        \cref{lem:lb_sc_roundreg};
        \begin{align*}
            \sum_{t=1}^\tau \E[f_t(w_t)] - f_t(w_t^*) 
            \geq \sum_{l=1}^{\tau/8} 
                \E[f_{t_l}(w_{t_l})] - f_{t_l}(w_{t_l}^*)
            \geq \frac{1}{2} \sum_{l=1}^{\tau/8} \frac{\sigma_{t_l}^2}{2}
            \geq \tau \frac{(\tau\delta)^2}{32^3 a^2 \zeta^2},
        \end{align*}
        and the proof is complete.
    \end{proof}

\begin{proof}[of \cref{lem:lbsc_main}]
    If $\zeta \geq C_0/2\cTV$, the result follows from \cref{lem:lbsc_pathweak}.
    Otherwise, note that $|\mu_1 - \mu_{\tau+1}| \geq \tau\delta$, and therefore by \cref{assume:dists};
    \[
        \sigma_1 + \sigma_{\tau+1}
        \geq \min\Big\{ 
            \frac{1}{C_1}, 
            \frac{\tau\delta}{\cTV \zeta} 
        \Big\}.
    \]
    Now, assume w.l.o.g.~that $\sigma_1 \geq \sigma_{\tau+1}$, then again by \cref{assume:dists} we have that for all $t\in\{2, \ldots, \tau\}$; $\sigma_t \geq \sigma_1$, or $\sigma_t \geq 1$, or
    \[
        \zeta \geq \frac{C_0}{C_1} - \frac{\sigma_t}{\sigma_1}
        \iff
        \frac{\sigma_t}{\sigma_1} \geq \frac{C_0}{C_1} - \zeta 
        \geq \frac{C_0}{2 C_1}.
    \]
    In any case, it follows that 
    \[ 
        \sigma_t \geq 
        \min\Big\{1, \frac{C_0}{2 C_1} \sigma_1 \Big\}
        \geq
        \frac{C_0}{2C_1^2} 
        \min\Big\{ 
            1, 
            \frac{\tau\delta}{\zeta} 
        \Big\}.
    \]
    Therefore by \cref{lem:lb_sc_roundreg};
    \begin{align*}
        \sum_{t=1}^\tau \E[f_t(w_t)] - f_t(w_t^*) 
        \geq \sum_{t=2}^\tau \frac{\sigma_t^2}{2}
        \geq (\tau-1) \frac{C_0^2}{8C_1^4} 
        \min\Big\{ 
            1, 
            \frac{(\tau\delta)^2}{\zeta^2} 
        \Big\}
        \geq \tau \frac{C_0^2}{16 C_1^4}
        \min\Big\{1, \frac{(\tau\delta)^2}{\zeta^2} \Big\}
        ,
    \end{align*}
    and we are done.
\end{proof}

\begin{proof} [of \cref{lem:lbsc_kkt}]
    Since the problem involves only linear equality constraints, any optimal solution must satisfy the KKT conditions.
    Consider the Lagrangian of the problem;
    \begin{align*}
        \mathcal L(\alpha, \beta, \gamma_1, \gamma_2) 
        = \sum_{k=1}^K \frac{\alpha_k^3}{\beta_k^2}
        + \gamma_1 \Big( \sum_{k=1}^K \alpha_k - 1\Big)
        + \gamma_2 \Big( \sum_{k=1}^K \beta_k - 1\Big).
    \end{align*}
    Let $\alpha^*, \beta^*$ be any optimal solution, then
    by the KKT conditions it follows that for all $k \leq K$
    \begin{align*}
        0 = \frac{\partial \mathcal L}{\partial \beta_k}
        = -2 \frac{(\alpha_k^*)^3}{(\beta_k^*)^3} + \gamma_2
        \implies
        \alpha_k^* = (\gamma_2/2)^{1/3} \beta_k^*.
    \end{align*}
    By the problem constraints, this further implies $\gamma_2 = 2$ and therefore $\alpha_k^* = \beta_k^*$ for all $k \leq K$.
    Hence $\sum_{k=1}^K \frac{(\alpha_k^*)^3}{(\beta_k^*)^2} 
    = \sum_{k=1}^K \alpha_k = 1$, which completes the proof.
  \end{proof}

\subsubsection[Assumption discussion]{Discussion of the applicability of \cref{assume:dists} to common distributions}
\label{sec:assume_dists_discuss}

In this section, we show \cref{assume:dists} holds for a general class of ``nice'' distributions. 
Algorithms in the spirit of \cref{alg:lftprl} we have presented in this work make use of such ``nice'' distributions as defined below, followed by an orthogonal projection onto the decision set. (Note, however, that the arguments in this section do not establish \cref{assume:dists} holds after the projection operation.)
Notably though, an immediate implication of the proof of \cref{thm:lbsc_oblivious} is that any follow-the-leader algorithm (that may not satisfy \cref{assume:dists}), is subject to the $\Omega(T/S^2)$ lower bound, and therefore in particular \cref{alg:lftprl}.
We proceed now with the formal definition of ``nice'' distributions and follow with arguments linking them to \cref{assume:dists}.

\begin{definition}\label{def:nice_dists}
    We say $\mathfrak P = \mathfrak P(C_\mu, C_\sigma)$ is a family of ``nice'' distributions with constants $C_\sigma \geq C_\mu > 0$, if the following holds.
    Let $p \in \mathfrak P$,
    denote $X \sim p$, $\mu_p \eqq \E X$, $\sigma_p^2 \eqq \E (X - \mu)^2$, then
\begin{enumerate}[label=(\roman*)]
    \item $p$ is symmetric; 
    $
        p(\mu_p - x) = p(\mu_p + x) \quad \forall x > 0.
    $
    \item For $x\in [\mu_p-\sigma_p, \mu_p + \sigma_p]$, it holds that 
        $
            \frac{1}{C_\sigma \sigma_p} 
            \leq p(x) \leq 
            \frac{1}{C_\mu \sigma_p}
        $.
\end{enumerate}
\end{definition} 

It is not hard to see the Normal, Laplace, and Uniform are examples of families that satisfy the above assumptions with an appropriate choice of constants $C_\mu, C_\sigma$. 
Proceeding, let $\mathfrak P = \mathfrak P(C_\mu, C_\sigma)$ be a family of ``nice'' distributions.
Our first lemma establishes the first part of \cref{assume:dists}.
\begin{lemma}
    For any $p_1, p_2 \in \mathfrak P$, it holds that
    \[
        \TV{p_1 - p_2} 
        \geq
        \frac{1}{2 C_\sigma}
        \min\Big\{ 1,
            \frac{|\mu_1 - \mu_2|}{\sigma_1 + \sigma_2}
        \Big\}. 
    \]
\end{lemma}
\begin{proof}
Let $p_1, p_2 \in \mathfrak P$, and
assume w.l.o.g.~$\mu_1 \leq \mu_2$. 
We have
\begin{align*}
    2\TV{p_1 - p_2} 
    \geq 
    \int_{\mu_1}^\infty |p_1 - p_2| 
    \geq 
    \int_{\mu_1}^\infty p_2 
        - \int_{\mu_1}^\infty p_1
    = \int_{\mu_1}^{\mu_2} p_2,
\end{align*}
where the last equality follows by the symmetry assumption $(i)$.
In addition, by property $(ii)$ we have that
    \[
        \int_{\mu_1}^{\mu_2} p_2 
        \geq \min\Big\{
            \frac{\mu_2 - \mu_1}{C_\sigma \sigma_2},
            \frac{1}{C_\sigma}
        \Big\},
    \]
    and the result follows.
\end{proof}

The second part of \cref{assume:dists} lower bounds the total variation with relation to the variances, regardless of the distance between means. The below lemma establishes it holds for $\mathfrak P$ with the appropriate choice of constants.

\begin{lemma}
    For any $p_1, p_2 \in \mathfrak P$ with variances $\sigma_1^2 \leq \sigma_2^2$, it holds that
    \[
        \TV{p_1 - p_2} 
        \geq
        \frac{1}{2 C_\mu C_\sigma} 
        \Big((C_\mu/C_\sigma) - \frac{\sigma_1}{\sigma_2})
    \]
\end{lemma}
\begin{proof}
Let $p_1, p_2 \in \mathfrak P$, with means $\mu_1, \mu_2 \in \R$ and variances $\sigma_1^2 \leq \sigma_2^2$. By property  $(ii)$ of \cref{def:nice_dists}, we have
\begin{align*}
    2\TV{p_1 - p_2} 
    \geq 
    \int_{\mu_1}^{\mu_1 + \sigma_1} |p_1 - p_2| 
    \geq
    \sigma_1 \Big(
        \frac{1}{C_\sigma \sigma_1} - \frac{1}{C_\mu \sigma_2}
    \Big)
    \\
    =
    \sigma_1 \Big(
        \frac{C_\mu \sigma_2 - C_\sigma \sigma_1}{C_\mu C_\sigma \sigma_1\sigma_2} 
    \Big)
    =\frac{1}{C_\mu C_\sigma} 
        \Big((C_\mu/C_\sigma) - \frac{\sigma_1}{\sigma_2})
    ,
\end{align*}
which proves the result.
\end{proof}

%% file: doc_appendix_rectified.tex
\section{Lazy switching, costs, and budgets} \label{apdx:sccb}
In this section, we discuss the relation between three notions of switching related OCO; \Ssoco, $S$-\sbudg, and $c$-\scost. Similar to lazy OCO studied in this paper, the \sbudg setting explored by \cite{altschuler2018online} also limits the player to a given number of switches $S$, though the limit is applied to the actual number of switches and not the expected. The $c$-\scost variant on the other hand, charges the player a unit cost of $c$ for every switch. \cref{lem:cswitc_Sswitch} below establishes this is in fact equivalent to limiting the number of expected switches. 
Table \ref{table:soco_bounds} lists results of prior work in both forms - as a function of the switching cost $c$ and as a function of the number of switches $S$.

\paragraph{Lazy vs \scost.}
There is a natural correspondence between \scost and \soco algorithms, as long as their guarantees are given in reasonable parametric forms.
This is summarized by the following lemma.

\begin{lemma}\label[lem]{lem:cswitc_Sswitch}
    Denote the $c$-\scost\ regret by
    \[
        \Re_T^{(c)} := \sum_{t=1}^T f_t(w_t) + c*\Ind{w_t \neq w_{t-1}} - \min_{w\in\Dom} \sum_{t=1}^T f_t(w)
    \]
    \begin{enumerate}
        \item If $\Alg_c$ guarantees $\E \Re_T^{(c)}  = (Tc)^\alpha$, then it may be converted to a \Ssoco\ player $\Alg_S'$ with 
            $\E \Re_T   = \left( \frac{T}{S} \right)^\frac{\alpha}{1-\alpha} $.

        \item If $\Alg_S$ is an \Ssoco\ player with $\E \Re_T = \frac{T}{S^\gamma}$, then it may be converted to a $c$-\scost\ player $\Alg_c'$ with $\E  \Re_T^{(c)}   = T^{\frac{1}{1+\gamma}} c^{\frac{\gamma}{1+\gamma}}$
    \end{enumerate}
\end{lemma}
\begin{proof}
    Assume an online player $\Alg_S$ is a \Ssoco\ algorithm with an expected regret guarantees of $R(T, S)$. Then reducing $\Alg_S$ to a $c$-\scost\ algorithm $\Alg_c'$ involves choosing $S = S(c)$ so as to minimize
    \[
        \Exp{ \Re_T^{(c)}(\Alg_c')} = R(T, S) + c S 
    \]
    To that end choose $S=S(c)$ to satisfy $S = \frac{R(T, S)}{c}$, which gives
    \[
        R(T, c) = 2*c*S(c).
    \]

    For the other direction, assume $\Alg_c$ has $R(T, c)$ expected regret guarantee in the $c$-\scost\ setting. We reduce $\Alg_c$ to a \Ssoco\ algorithm $\Alg_S'$ by solving $c S = R(T, c)$ for $c=c(S)$ and running $\Alg_c(c(S))$. This guarantees
    \[
        \Exp{\Sw(\Alg_S') } \leq S \quad \textit{ and } \quad \Exp{\Re_T(\Alg_S')} \leq c(S)*S.
    \]

    In particular, if $\Alg_c$ has a $R(T, c) = (Tc)^\alpha$ guarantee then 
    \[
        c S = T^\alpha c^\alpha \implies c = \left(\frac{T^\alpha}{S}\right)^{\frac{1}{1-\alpha}},
    \]
    and the reduction gives $\Alg_S'$ with  
    \[
        \Exp{\Sw(\Alg_S')} \leq S \quad \textit{ and } \quad \Exp{\Re_T(\Alg_S')} 
        \leq S \left(\frac{T^\alpha}{S}\right)^{\frac{1}{1-\alpha}} 
        = \left( \frac{T}{S} \right)^\frac{\alpha}{1-\alpha}.
    \]
    Similarly, if $\Alg_S$ has a $R(T, S) = \frac{T}{S^\gamma}$ guarantee, the reduction gives $\Alg_c'$ with 
    \[
        \Exp{ \Re_T^{(c)}(\Alg_c')} = T^{\frac{1}{1+\gamma}} c^{\frac{\gamma}{1+\gamma}}.
        \qedhere
    \]
   
\end{proof}

\begin{table}[ht]
    \bgroup
    \def\arraystretch{2.1}
    \begin{center}
        \begin{threeparttable}
        \caption{Lazy OCO / \scost bounds from previous work ($c \geq 1$).}
        \begin{tabular}{ c | c | c | c | c}
            \textbf{Setting}
            & \textbf{Adversary}
            & \makecell{ \textbf{S-Lazy} \textbf{OCO} }
            & \makecell{ \textbf{c-Switching} \\ \textbf{Cost} }
            & \textbf{Reference}
            \\ 
            \hline
            MAB 
                & Oblivious
                & \makecell{
                    $ O( \frac{T} {\sqrt {S}} ) $ \tnote{b}
                }
                & \makecell{
                        $O(T^{2/3} c^{1/3})$
                }
                & \makecell{
                        \cite{arora2012online}
                }
            \\ \hline
            MAB 
                & Oblivious
                & \makecell{
                    $ \tilde \Omega( \frac{T} {\sqrt {S}} ) $ \tnote{b}
                }
                & \makecell{
                    $\tilde \Omega(T^{2/3} c^{1/3})$
                }
                & \makecell{
                        \cite{dekel2014bandits}
                }
            \\
            \hline
            Experts 
                & Oblivious
                & \makecell{
                    $ O(\frac{T}{S}) $ \tnote{a} 
                }
                & \makecell{
                    $O(\sqrt {T c})$
                }
                & \makecell{
                    \cite{kalai2005efficient}
                    \\ 
                    \cite{geulen2010regret}
                    \\
                    \cite{devroye2013prediction}
                }
            \\ \hline
            Experts  
                & Oblivious
                & \makecell{
                    $ \tilde \Theta(\frac{T}{S}) $ \tnote{a, b}
                }
                & \makecell{
                    -
                }
                & \makecell{
                    \cite{altschuler2018online}
                }
            \\ \hline
            OCO
                & Oblivious
                & \makecell{
                    $ O(\frac{dT}{S})$ \tnote{a} 
                }
                & \makecell{
                    $O(d \sqrt {T c})$
                }
                & \makecell{
                        \cite{anava2015online}
                }
            \\
            \hline
            OCO
                & Adaptive
                & \makecell{
                    $ \Theta(\frac{T}{\sqrt {S} }) $ \tnote{b}
                }
                & \makecell{
                    $ \Theta( T^{2/3} c^{1/3} )$
                }
                & \makecell{
                        \cite{chen2019minimax}
                }
            \\
       \end{tabular}
       
       \label{table:soco_bounds}
       \begin{tablenotes}\footnotesize
            \item[a] For $S = O \big(\sqrt T \big)$.
            \item[b] Also apply in the \sbudg setting.
        \end{tablenotes}
       \end{threeparttable}
    \end{center}
    \egroup
    \end{table}

\paragraph{Lazy vs \sbudg.} Recently studied in the work of \cite{altschuler2018online}, \sbudg\ regret is defined as the regret guarantee achievable by an algorithm under a hard cap budget of $S\in [T]$ switches, a limit that should be met on every game execution.
In general, while an $S$-\sbudg player is of course also an \Ssoco one, the converse is not necessarily true. However, when the adversary is adaptive, the two notions are equivalent. 
To see this, consider the adaptive adversary given in \citet[Proposition 6]{chen2019minimax}. This adversary ensures $\Omega(T/\sqrt {\Sw_T})$ regret against any player that makes $\Sw_T$ switches, and furthermore, does not need to know the number of switches in advance. In particular, this adversary is completely unaffected by any randomness employed by the player. Therefore, By Jensen's inequality and convexity of $x \mapsto 1/\sqrt{x}$, we have that $\mathbb E \Re_T \geq T/\sqrt{S}$ when $\mathbb E \Sw_T = S$. In addition, a deterministic player employing the blocking technique ensures $\Re_T = O(T/\sqrt S)$. Thus, we have that the minimax switching regret of both $S$-\sbudg and \Ssoco is $\Theta(T / \sqrt S)$ in the adaptive adversary setting.

\section{Sampling from maximal couplings} 
\label{rej_appendix}
 In this section, we provide proofs for \cref{sec:max_couplings}.
 
\begin{proof} [of Lemma \ref{lem:lazysample-switch}]
    The algorithm samples from $\dist P$ when it reaches the loop, which happens if $\dist P(x) < z$ for $z \sim \dist U[0, \dist Q(x)]$. Denote
    \[
        B := \left\{ w \in \R^d \mid \dist P(w) > \dist Q(w) \right\},
    \]
    and observe
    \begin{align*}
        \Pr\left( \dist P \text{ is sampled from} \right) 
            &= \Pr(\dist P(x) < z)
            \\
            &= \int_{x \notin B} \dist Q(x) \left( 1 - \frac{\dist P(x)}{ \dist Q(x)}\right) dx
            \\
            &= \int_{x \notin B} \dist Q(x) - \dist P(x) dx
            \\
            &= \norm{\dist Q - \dist P}_{TV}
    \end{align*}
\end{proof}

\begin{proof} [of Lemma \ref{lem:lazysample-dist}]
    Set
    \[
        \rv W := \text{LazySample} (\rv X, \dist Q, \dist P); \quad \rv X \sim \dist Q,
    \]
    and we aim to prove that
    \[
        \Pr (\rv W = w) = \dist P(w).
    \]
    We analyse each of the two cases $P(w) > \dist Q(w)$ and $P(w) < \dist Q(w)$ separately. Note $w$ is not random, it is just the sample point at which we want to show the two densities are equal. Proceeding, denote 
    \[
        B := \left\{ w \in \R^d \mid \dist P(w) > \dist Q(w) \right\},
    \]
    and first assume $w \notin B$. In this case it must be that $w$ was returned by the first return statement, and thus equals $x$. We have
    \begin{align*}
        \Pr(\rv W = w) 
            &= \Pr(\rv W = w \wedge \rv X \in B) + \Pr(\rv W = w \wedge \rv X \notin B)
            \\
            &= \Pr(\rv W = w \wedge \rv X \notin B)
            \\
            &= \Pr(\rv X = w \wedge \dist P(w) < z) \quad (\text{where } z \sim \dist U[0, \dist Q(x)])
            \\
            &= \dist Q(w) \frac{\dist P(w)}{\dist Q(w)}
            \\
            &= \dist P(w).
    \end{align*}
    In addition, the running time is clearly $O(1)$.
    Now consider the case that $w \in B$. As before, we have
    \begin{align*}
        \Pr(\rv W = w) 
            &= \Pr(\rv W = w \wedge \rv X \in B) + \Pr(\rv W = w \wedge \rv X \notin B),
    \end{align*}
    and
    \[
        \Pr(\rv W = w \wedge \rv X \in B) = \Pr(\rv X = w) = \dist Q(w).
    \]
    First, observe that on each iteration the probability to exit the loop is
    \[
        \integ{B}{ \dist P(t)\left(1 - \frac{\dist Q(t)}{\dist P(t)} \right) }{t} = \norm{\dist P - \dist Q}_{TV}.
    \]
    Therefore, we are expected to exit it in time $O(\norm{\dist P - \dist Q}_{TV}^{-1})$, and together with \cref{lem:lazysample-switch} this implies the $O(1)$ running time.
    To compute $\Pr(\rv W = w \wedge \rv X \notin B)$, denote the density of the value produced with the loop by $\nu(y)$, and observe that for $y\in B$
    \[
        \nu(y) = \frac{
            \dist P(y)\left(1 - \frac{\dist Q(y)}{\dist P(y)} \right)
        } {
            \integ{B}{ \dist P(t)\left(1 - \frac{\dist Q(t)}{\dist P(t)} \right) }{t}
        } = \frac{
            \dist P (y) - \dist Q(y)
        }{
            \dist P (B) - \dist Q(B)
        }.
    \]
    Therefore
    \begin{align*}
        \Pr(\rv W = w \wedge \rv X \notin B) 
            &= \int_{x \notin B} \dist Q(x) \Pr(\rv W = w \mid \rv X = x) dx
            \\
            &= \int_{x \notin B} \dist Q(x)\left(1 - \frac{\dist P(x)}{ \dist Q(x)}\right) \nu(w) dx
            \\
            &= \frac{\dist P (w) - \dist Q(w)}{\dist P (B) - \dist Q(B)} \int_{x \notin B} \dist Q(x) - \dist P(x) dx
            \\
            &= \frac{\dist P (w) - \dist Q(w)}{\dist P (B) - \dist Q(B)} \left( \dist P (B) - \dist Q(B) \right)
            \\
            &= \dist P (w) - \dist Q(w).
    \end{align*}
    
    All in all, we have obtained
    \begin{align*}
        \Pr(\rv W = w) 
            = \dist Q(w) + \dist P (w) - \dist Q(w)
            = \dist P(w),
    \end{align*}
    which concludes the case of $w \in B$ and therefore completes the proof.
\end{proof}

\section[Total variation bound proof]{Proof of \cref{lem:tv_decisions}}
\label{sec:lem:tv_decisions}

We begin with establishing a closed form expression for the density function $\dist Q_t$. 
The proofs below are based on arguments given in \cite{agarwal2023differentially}.
\begin{lemma}\label{lem:cov}
    For any $w\in W$, we have
    \begin{align*}
        \cQ_t(w) = \nu_t(-\nabla \phi_t(w) )
            \av{\det \br{ -\nabla^2 \phi_t(w)}},
    \end{align*}
    where $\phi_t(w) \eqq \phi_t(w; 0)$,
    and $\nu_t$ denotes the density function of $p_t \sim {\rm Lap}(\sigma_t)$ (see \cref{eq:def_laplace}).
\end{lemma}
\begin{proof}
    By the fact that the objective $\phi_t$ includes the Legendre function $B$, we have that $w_t(p) \in {\rm int}(W)$ for all $p \in \R^d$.
	Further, by differentiability of $\phi_t$ and optimality conditions, we have;
	\begin{align*}
		0 = \nabla \phi_t(w_t(p); p) = \nabla \phi_t(w_t(p)) + p
		\implies 
		- \nabla \phi_t(w_t(p)) = p.
	\end{align*}
	By strict convexity of $\phi_t$, the mapping $w \mapsto  - \nabla \phi_t(w)$ is one-to-one, with its inverse given by
	$
		w_t(p) = \nabla \phi_t^\star (-p)
	$
	where $\phi_t^\star$ denotes the Fenchel conjugate of $\phi_t$.
	The result now follows
	by the change of variables formula \citep[e.g.,][]{bogachev2007measure}.
\end{proof}

\begin{proof}[of \cref{lem:tv_decisions}]
	First, we establish a uniform bound on the density ratio $\frac{\cQ_{t+1}(w)}{\cQ_t(w)}$. 
	By \cref{lem:cov}, we have
	\begin{align*}
		\frac{\cQ_{t+1}(w)}{\cQ_t(w)} 
		= \frac{\nu_{t+1}(-\nabla \phi_{t+1}(w) )}
			{\nu_t(-\nabla \phi_t(w) )}
		\cdot \av{\frac{\det (-\nabla^2 \phi_{t+1}(w)) }
			{\det (-\nabla^2 \phi_t(w))}
		}.
	\end{align*}
	To bound the second term, note that since $f_t$ is $\beta$-smooth,
	\begin{align*}
		\norm{\nabla^2 \phi_t(w) - \nabla^2 \phi_{t+1}(w)}
		= \norm{\nabla^2 f_{t}(w)} \leq \beta.
	\end{align*}
	Further, letting $\lambda_{t, i}$ denote that $i$'th largest eisgenvalue of $\nabla^2 \phi_t(w)$, this
	implies $\av{\lambda_{t, i} - \lambda_{t+1, i} } \leq \beta$ for all $i\in[d]$ \citep{bhatia2013matrix}. 
	In addition, using that $\phi_t$ is $(1/\eta_t)$-strongly convex, we have $\lambda_{t, i} \geq 1/\eta_t$, hence;
	\begin{align*}
		\av{\frac{\det (-\nabla^2 \phi_{t+1}(w)) }
			{\det (-\nabla^2 \phi_t(w))}
		} \leq \prod_{i=1}^d \frac{\lambda_{i,t} + \beta}{\lambda_{i,t}}
		\leq \prod_{i=1}^d (1 + \eta_t\beta)
		\leq e^{\eta_t \beta d}.
	\end{align*}
	For the first term, 
	\begin{align*}
		\frac{\nu_{t+1}(-\nabla \phi_{t+1}(w) )}
			{\nu_{t}(-\nabla \phi_t(w) )}
		= \frac{
			\exp \br{
			-\norm{\nabla \phi_{t+1}(w)}_1/\sigma_t
		}}{
			\exp \br{
			-\norm{\nabla \phi_t(w)}_1/\sigma_{t}
		}}
		&= \exp \br{
			(\norm{\nabla \phi_t(w)}_1 -\norm{\nabla \phi_{t+1}(w)}_1) / \sigma_{t}
			}
		\\
		&\leq \exp \br{
			\norm{\nabla \phi_t(w) - \nabla \phi_{t+1}(w)}_1/\sigma_{t}
			}
		\\
		&= \exp \br{
			\norm{\nabla f_{t}(w)}_1/\sigma_{t}
			}
		\\
		&\leq e^{\sqrt d G /\sigma_t}.
	\end{align*}
	Proceeding, we let $\delta_t \eqq \eta_t \beta d + \sqrt d G/\sigma$, and note the above implies, for all $w$;
	\begin{align*}
    	\frac{\cQ_{t+1}(w)}{\cQ_{t}(w)} 
		&\leq e^{\delta_t}
		\\
		\iff
		\frac{\cQ_{t}(w)}{\cQ_{t+1}(w)} 
		&\geq e^{-\delta_t}
		\geq 1 - \delta_t
		\\
		\iff
		\cQ_{t}(w)
		&\geq \cQ_{t+1}(w) - \delta_t \cQ_{t+1}(w)
    \end{align*}
    Thus, for any set $E \subseteq {\rm int} (W)$;
    \begin{align*}
    	\cQ_{t+1}(E) - \cQ_{t}(E)
    	\leq 
    	\delta_t \cQ_{t+1}(E) \leq \delta_t.
    \end{align*}
    The other direction follows from identical symmetric arguments, and establishes
    \begin{align*}
    	\max_{E \in {\rm int}(W)} \av{\cQ_{t+1}(E) - \cQ_{t}(E)} 
    	\leq \delta_t = \eta_t \beta d + \sqrt d G/\sigma,
    \end{align*}
which completes the proof.
\end{proof}

\section{Technical lemmas}
\label{sec:lem_tech}

\begin{proof} [of \cref{lem:varlb}]
    We have
    \begin{align*}
        (\mu_p - \mu_q)^2 
        &= \left (\sum_x \big(p(x) - q(x) \big) x \right)^2
        \\
        &= \left (\sum_x \big(p(x) - q(x) \big) (x-\mu_p) \right)^2
        \\
        &\leq \left (\sum_x \big|p(x) - q(x) \big| (x-\mu_p) \right)^2
        \\
        &= \left (\sum_x \sqrt{\big|p(x) - q(x)\big|} \sqrt{\big|p(x) - q(x)\big|}  (x-\mu_p) \right)^2
        \\
        &\leq \left( \sum_x \big|p(x) - q(x)\big| \right)
            \left( \sum_x \big|p(x) - q(x)\big| (x-\mu_p)^2 \right)
        \\
        &= 2 \norm{p - q}_{TV} \left( \sum_x \big|p(x) - q(x)\big| (x-\mu_p)^2 \right).
    \end{align*}
    In addition;
    \begin{align*}
        & \sum_x \big|p(x) - q(x)\big| (x-\mu_p)^2
        \\
        &= \sum_{x: p(x) > q(x)} (p(x) - q(x)) (x-\mu_p)^2
            + \sum_{x: q(x) > p(x)} (q(x) - p(x)) (x-\mu_p)^2
        \\
        &\leq \sum_{x} p(x) (x-\mu_p)^2 + \sum_{x} q(x) (x-\mu_p)^2
        \\
        &= \sigma_p^2 + \sum_{x} q(x) (x-\mu_q)^2 + \sum_{x} q(x) (\mu_p-\mu_q)^2
        \\
        &= \sigma_p^2 + \sigma_q^2 + (\mu_p-\mu_q)^2.
    \end{align*}
    To conclude we have shown that
    \[
        \frac{(\mu_p - \mu_q)^2 }{2 \norm{p - q}_{TV}} \leq \sigma_p^2 + \sigma_q^2 + (\mu_p-\mu_q)^2.
        \qedhere
    \]
\end{proof}

\begin{proof}[of \cref{lem:max_coupling}]
    Let $x\in \Omega$, and observe
    \begin{align*}
        \Pr(\rv Y = x) &\geq \Pr(\rv Y = x \wedge \rv X = x)
        \\
        &= \Pr(\rv X = x) \Pr(\rv Y = x \mid \rv X = x)
        \\
        &= \Pr(\rv X = x) (1 - \Pr(\rv Y \neq x \mid \rv X = x)).
    \end{align*}
    Rearranging the above, we obtain
    \begin{align*}
        \Pr(\rv X = x) \Pr(\rv Y \neq x \mid \rv X = x) 
        \geq \Pr(\rv X = x)  - \Pr(\rv Y = x),
    \end{align*}
    meaning $\Pr(\rv X = x \wedge \rv Y \neq x)
        \geq p(x)  - q(x)$.
    Now
    \begin{align*}
        \Pr(\rv X \neq \rv Y) 
        &= \sum_{x \in \Omega} \Pr(\rv X = x \wedge \rv Y \neq x)
        \\
        &\geq \sum_{x: p(x) > q(x)} \Pr(\rv X = x \wedge \rv Y \neq x)
        \\
        &\geq \sum_{x: p(x) > q(x)} p(x) - q(x) 
        \\
        &= \norm{p - q}_{TV}. \qedhere
    \end{align*}
\end{proof}

\begin{proof}[of \cref{lem:ber_tv} (based on \citealp{slivkins2019introduction})]
    By Pinsker's Inequality and chain rule of the KL-divergence we have
    \begin{align*}
       \norm{\Ber_{p}^j - \Ber_{q}^j}_{TV}^2
        &\leq \frac{1}{2} \KL{\Ber_p^j}{\Ber_q^j}
        \\
        &= \frac{j}{2} \KL{\Ber_p}{\Ber_q}.
    \end{align*}
    In addition;
    \begin{align*}
        \KL{\Ber_p}{\Ber_q}
        &= \frac{1+\epsilon}{2}\log\frac{1+\epsilon}{1-\epsilon}
            + \frac{1-\epsilon}{2}\log\frac{1-\epsilon}{1+\epsilon}
        \\
        &= \frac{1}{2}\left(
            \log\frac{1+\epsilon}{1-\epsilon} + \log\frac{1-\epsilon}{1+\epsilon}
        \right)
            + \frac{\epsilon}{2} \left(
            \log\frac{1+\epsilon}{1-\epsilon} - \log\frac{1-\epsilon}{1+\epsilon}
        \right)
        \\
        &= \frac{\epsilon}{2} 
            \log\frac{(1+\epsilon)^2}{(1-\epsilon)^2}
        \\
        &= \frac{\epsilon}{2} 
            \log\frac{
                (1-\epsilon)^2 + 4\epsilon(1-\epsilon) + 4\epsilon^2
                }{(1-\epsilon)^2}
        \\
        &= \frac{\epsilon}{2} \log\left(
            1 + \frac{4\epsilon}{(1-\epsilon^2)}
        \right)
        \\
        &\leq \frac{\epsilon}{2} \log\left(
            1 + 8\epsilon
        \right)
        \\
        &\leq 4 \epsilon^2. \qedhere
    \end{align*}
\end{proof}

\begin{lemma}\label[lem]{lem:sc_stab}
    Let $\phi: \Dom \to \R$ be $\lambda$-strongly convex, $h:\Dom \to \R$ differentiable, and set 
    $x^* = \argmin_{x\in \Dom} \big \{\phi(x) + h(x) \big\}$, 
    $x' = \argmin_{x\in \Dom} \phi(x)$.
    It holds that
    \[
        \norm{x^* - x'} 
        \leq \frac{2\norm{\nabla h(x')}}{\lambda}.
    \]
\end{lemma}
\begin{proof} 
    Denote $\psi(x) = \phi(x) + h(x)$. We have
    \[
        \frac{\lambda}{2}\norm{x' - x^*}^2 
            \leq \psi(x') - \psi(x^*)
            \leq \nabla \psi(x')\T (x' - x^*).
    \]
    In addition, $x'$ minimizes $\psi - h$ over $\Dom$, therefore
    \begin{align*}
        0 &\leq (\nabla \psi(x') - \nabla h(x') )\T (x^* - x').
        \\ 
        \iff
        \nabla \psi(x')\T (x' - x^*)
            &\leq \nabla h(x') \T (x' - x^*) 
            \leq \norm{\nabla h(x')} \norm{x' - x^*}.
    \end{align*}
    Combining the above inequalities we obtain
    \[
        \frac{\lambda}{2}\norm{x' - x^*}^2 
            \leq \norm{\nabla h(x')} \norm{x' - x^*},
    \]
    and the result follows after dividing both sides by $(\lambda/2) \norm{x' - x^*}$.
\end{proof}

%% file: main-arxiv.bbl
\begin{thebibliography}{52}
\providecommand{\natexlab}[1]{#1}
\providecommand{\url}[1]{\texttt{#1}}
\expandafter\ifx\csname urlstyle\endcsname\relax
  \providecommand{\doi}[1]{doi: #1}\else
  \providecommand{\doi}{doi: \begingroup \urlstyle{rm}\Url}\fi

\bibitem[Agarwal and Singh(2017)]{agarwal2017price}
N.~Agarwal and K.~Singh.
\newblock The price of differential privacy for online learning.
\newblock In \emph{International Conference on Machine Learning}, pages 32--40.
  PMLR, 2017.

\bibitem[Agarwal et~al.(2019{\natexlab{a}})Agarwal, Bullins, Hazan, Kakade, and
  Singh]{agarwal2019online}
N.~Agarwal, B.~Bullins, E.~Hazan, S.~Kakade, and K.~Singh.
\newblock Online control with adversarial disturbances.
\newblock In \emph{International Conference on Machine Learning}, pages
  111--119. PMLR, 2019{\natexlab{a}}.

\bibitem[Agarwal et~al.(2019{\natexlab{b}})Agarwal, Hazan, and
  Singh]{agarwal2019logarithmic}
N.~Agarwal, E.~Hazan, and K.~Singh.
\newblock Logarithmic regret for online control.
\newblock In \emph{Advances in Neural Information Processing Systems}, pages
  10175--10184, 2019{\natexlab{b}}.

\bibitem[Agarwal et~al.(2023)Agarwal, Kale, Singh, and
  Thakurta]{agarwal2023differentially}
N.~Agarwal, S.~Kale, K.~Singh, and A.~Thakurta.
\newblock Differentially private and lazy online convex optimization.
\newblock In \emph{The Thirty Sixth Annual Conference on Learning Theory},
  pages 4599--4632. PMLR, 2023.

\bibitem[Altschuler and Talwar(2018)]{altschuler2018online}
J.~Altschuler and K.~Talwar.
\newblock Online learning over a finite action set with limited switching.
\newblock In \emph{Conference On Learning Theory}, pages 1569--1573. PMLR,
  2018.

\bibitem[Anava et~al.(2015)Anava, Hazan, and Mannor]{anava2015online}
O.~Anava, E.~Hazan, and S.~Mannor.
\newblock Online learning for adversaries with memory: price of past mistakes.
\newblock In \emph{Advances in Neural Information Processing Systems}, pages
  784--792, 2015.

\bibitem[Andrew et~al.(2013)Andrew, Barman, Ligett, Lin, Meyerson, Roytman, and
  Wierman]{andrew2013tale}
L.~Andrew, S.~Barman, K.~Ligett, M.~Lin, A.~Meyerson, A.~Roytman, and
  A.~Wierman.
\newblock A tale of two metrics: Simultaneous bounds on competitiveness and
  regret.
\newblock In \emph{Conference on Learning Theory}, pages 741--763. PMLR, 2013.

\bibitem[Antoniadis and Schewior(2017)]{antoniadis2017tight}
A.~Antoniadis and K.~Schewior.
\newblock A tight lower bound for online convex optimization with switching
  costs.
\newblock In \emph{International Workshop on Approximation and Online
  Algorithms}, pages 164--175. Springer, 2017.

\bibitem[Arora et~al.(2012)Arora, Dekel, and Tewari]{arora2012online}
R.~Arora, O.~Dekel, and A.~Tewari.
\newblock Online bandit learning against an adaptive adversary: from regret to
  policy regret.
\newblock In \emph{Proceedings of the 29th International Coference on
  International Conference on Machine Learning}, pages 1747--1754, 2012.

\bibitem[Asi et~al.(2023{\natexlab{a}})Asi, Feldman, Koren, and
  Talwar]{asi2023near}
H.~Asi, V.~Feldman, T.~Koren, and K.~Talwar.
\newblock Near-optimal algorithms for private online optimization in the
  realizable regime.
\newblock \emph{arXiv preprint arXiv:2302.14154}, 2023{\natexlab{a}}.

\bibitem[Asi et~al.(2023{\natexlab{b}})Asi, Feldman, Koren, and
  Talwar]{asi2023private}
H.~Asi, V.~Feldman, T.~Koren, and K.~Talwar.
\newblock Private online prediction from experts: Separations and faster rates.
\newblock In \emph{The Thirty Sixth Annual Conference on Learning Theory},
  pages 674--699. PMLR, 2023{\natexlab{b}}.

\bibitem[Auer et~al.(2002)Auer, Cesa-Bianchi, Freund, and
  Schapire]{auer2002nonstochastic}
P.~Auer, N.~Cesa-Bianchi, Y.~Freund, and R.~E. Schapire.
\newblock The nonstochastic multiarmed bandit problem.
\newblock \emph{SIAM journal on computing}, 32\penalty0 (1):\penalty0 48--77,
  2002.

\bibitem[Awerbuch and Kleinberg(2008)]{awerbuch2008online}
B.~Awerbuch and R.~Kleinberg.
\newblock Online linear optimization and adaptive routing.
\newblock \emph{Journal of Computer and System Sciences}, 74\penalty0
  (1):\penalty0 97--114, 2008.

\bibitem[Bansal et~al.(2015)Bansal, Gupta, Krishnaswamy, Pruhs, Schewior, and
  Stein]{bansal20152}
N.~Bansal, A.~Gupta, R.~Krishnaswamy, K.~Pruhs, K.~Schewior, and C.~Stein.
\newblock A 2-competitive algorithm for online convex optimization with
  switching costs.
\newblock In \emph{Approximation, Randomization, and Combinatorial
  Optimization. Algorithms and Techniques (APPROX/RANDOM 2015)}. Schloss
  Dagstuhl-Leibniz-Zentrum fuer Informatik, 2015.

\bibitem[Bavarian et~al.(2020)Bavarian, Ghazi, Haramaty, Kamath, Rivest, and
  Sudan]{bavarian2020optimality}
M.~Bavarian, B.~Ghazi, E.~Haramaty, P.~Kamath, R.~L. Rivest, and M.~Sudan.
\newblock Optimality of correlated sampling strategies.
\newblock \emph{Theory of Computing}, 16\penalty0 (1):\penalty0 1--18, 2020.

\bibitem[Bhatia(2013)]{bhatia2013matrix}
R.~Bhatia.
\newblock \emph{Matrix analysis}, volume 169.
\newblock Springer Science \& Business Media, 2013.

\bibitem[Blum and Burch(2000)]{blum2000line}
A.~Blum and C.~Burch.
\newblock On-line learning and the metrical task system problem.
\newblock \emph{Machine Learning}, 39\penalty0 (1):\penalty0 35--58, 2000.

\bibitem[Bogachev and Ruas(2007)]{bogachev2007measure}
V.~I. Bogachev and M.~A.~S. Ruas.
\newblock \emph{Measure theory}, volume~1.
\newblock Springer, 2007.

\bibitem[Borodin and El-Yaniv(2005)]{borodin2005online}
A.~Borodin and R.~El-Yaniv.
\newblock \emph{Online computation and competitive analysis}.
\newblock cambridge university press, 2005.

\bibitem[Borodin et~al.(1992)Borodin, Linial, and Saks]{borodin1992optimal}
A.~Borodin, N.~Linial, and M.~E. Saks.
\newblock An optimal on-line algorithm for metrical task system.
\newblock \emph{Journal of the ACM (JACM)}, 39\penalty0 (4):\penalty0 745--763,
  1992.

\bibitem[Broder(1997)]{broder1997resemblance}
A.~Z. Broder.
\newblock On the resemblance and containment of documents.
\newblock In \emph{Proceedings. Compression and Complexity of SEQUENCES 1997
  (Cat. No. 97TB100171)}, pages 21--29. IEEE, 1997.

\bibitem[Buchbinder et~al.(2012)Buchbinder, Chen, Naor, and
  Shamir]{buchbinder2012unified}
N.~Buchbinder, S.~Chen, J.~S. Naor, and O.~Shamir.
\newblock Unified algorithms for online learning and competitive analysis.
\newblock In \emph{Conference on Learning Theory}, pages 5--1. JMLR Workshop
  and Conference Proceedings, 2012.

\bibitem[Cesa-Bianchi and Lugosi(2006)]{cesa2006prediction}
N.~Cesa-Bianchi and G.~Lugosi.
\newblock \emph{Prediction, learning, and games}.
\newblock Cambridge university press, 2006.

\bibitem[Cesa-Bianchi et~al.(2013)Cesa-Bianchi, Dekel, and
  Shamir]{cesa2013online}
N.~Cesa-Bianchi, O.~Dekel, and O.~Shamir.
\newblock Online learning with switching costs and other adaptive adversaries.
\newblock In \emph{Proceedings of the 26th International Conference on Neural
  Information Processing Systems-Volume 1}, pages 1160--1168, 2013.

\bibitem[Chen et~al.(2019)Chen, Yu, Lawrence, and Karbasi]{chen2019minimax}
L.~Chen, Q.~Yu, H.~Lawrence, and A.~Karbasi.
\newblock Minimax regret of switching-constrained online convex optimization:
  No phase transition.
\newblock \emph{arXiv preprint arXiv:1910.10873}, 2019.

\bibitem[Chen et~al.(2018)Chen, Goel, and Wierman]{chen2018smoothed}
N.~Chen, G.~Goel, and A.~Wierman.
\newblock Smoothed online convex optimization in high dimensions via online
  balanced descent.
\newblock In \emph{Conference On Learning Theory}, pages 1574--1594. PMLR,
  2018.

\bibitem[Cohen et~al.(2018)Cohen, Hasidim, Koren, Lazic, Mansour, and
  Talwar]{cohen2018online}
A.~Cohen, A.~Hasidim, T.~Koren, N.~Lazic, Y.~Mansour, and K.~Talwar.
\newblock Online linear quadratic control.
\newblock In \emph{International Conference on Machine Learning}, pages
  1029--1038. PMLR, 2018.

\bibitem[Cohen et~al.(2019)Cohen, Koren, and Mansour]{cohen2019learning}
A.~Cohen, T.~Koren, and Y.~Mansour.
\newblock Learning linear-quadratic regulators efficiently with only $\sqrt{T}$
  regret.
\newblock In \emph{International Conference on Machine Learning}, pages
  1300--1309. PMLR, 2019.

\bibitem[Dekel et~al.(2014)Dekel, Ding, Koren, and Peres]{dekel2014bandits}
O.~Dekel, J.~Ding, T.~Koren, and Y.~Peres.
\newblock Bandits with switching costs: ${T}^{2/3}$ regret.
\newblock In \emph{Proceedings of the forty-sixth annual ACM symposium on
  Theory of computing}, pages 459--467, 2014.

\bibitem[Devroye et~al.(2013)Devroye, Lugosi, and Neu]{devroye2013prediction}
L.~Devroye, G.~Lugosi, and G.~Neu.
\newblock Prediction by random-walk perturbation.
\newblock In \emph{Conference on Learning Theory}, pages 460--473, 2013.

\bibitem[Devroye et~al.(2018)Devroye, Mehrabian, and Reddad]{devroye2018total}
L.~Devroye, A.~Mehrabian, and T.~Reddad.
\newblock The total variation distance between high-dimensional gaussians.
\newblock \emph{arXiv preprint arXiv:1810.08693}, 2018.

\bibitem[Dwork et~al.(2014)Dwork, Roth, et~al.]{dwork2014algorithmic}
C.~Dwork, A.~Roth, et~al.
\newblock The algorithmic foundations of differential privacy.
\newblock \emph{Foundations and Trends{\textregistered} in Theoretical Computer
  Science}, 9\penalty0 (3--4):\penalty0 211--407, 2014.

\bibitem[Feamster et~al.(2014)Feamster, Rexford, and Zegura]{feamster2014road}
N.~Feamster, J.~Rexford, and E.~Zegura.
\newblock The road to sdn: an intellectual history of programmable networks.
\newblock \emph{ACM SIGCOMM Computer Communication Review}, 44\penalty0
  (2):\penalty0 87--98, 2014.

\bibitem[Foster and Simchowitz(2020)]{foster2020logarithmic}
D.~Foster and M.~Simchowitz.
\newblock Logarithmic regret for adversarial online control.
\newblock In \emph{International Conference on Machine Learning}, pages
  3211--3221. PMLR, 2020.

\bibitem[Geulen et~al.(2010)Geulen, V{\"o}cking, and Winkler]{geulen2010regret}
S.~Geulen, B.~V{\"o}cking, and M.~Winkler.
\newblock Regret minimization for online buffering problems using the weighted
  majority algorithm.
\newblock In \emph{COLT}, pages 132--143. Citeseer, 2010.

\bibitem[Goel et~al.(2019)Goel, Lin, Sun, and Wierman]{goel2019beyond}
G.~Goel, Y.~Lin, H.~Sun, and A.~Wierman.
\newblock Beyond online balanced descent: An optimal algorithm for smoothed
  online optimization.
\newblock \emph{Advances in Neural Information Processing Systems},
  32:\penalty0 1875--1885, 2019.

\bibitem[Guha~Thakurta and Smith(2013)]{guha2013nearly}
A.~Guha~Thakurta and A.~Smith.
\newblock (nearly) optimal algorithms for private online learning in
  full-information and bandit settings.
\newblock \emph{Advances in Neural Information Processing Systems}, 26, 2013.

\bibitem[Hazan(2019)]{hazan2019introduction}
E.~Hazan.
\newblock Introduction to online convex optimization.
\newblock \emph{arXiv preprint arXiv:1909.05207}, 2019.

\bibitem[Jacob et~al.(2020)Jacob, O’Leary, and
  Atchad{\'e}]{jacob2020unbiased}
P.~E. Jacob, J.~O’Leary, and Y.~F. Atchad{\'e}.
\newblock Unbiased markov chain monte carlo methods with couplings.
\newblock \emph{Journal of the Royal Statistical Society: Series B (Statistical
  Methodology)}, 82\penalty0 (3):\penalty0 543--600, 2020.

\bibitem[Jaghargh et~al.(2019)Jaghargh, Krause, Lattanzi, and
  Vassilvtiskii]{jaghargh2019consistent}
M.~R.~K. Jaghargh, A.~Krause, S.~Lattanzi, and S.~Vassilvtiskii.
\newblock Consistent online optimization: Convex and submodular.
\newblock In \emph{The 22nd International Conference on Artificial Intelligence
  and Statistics}, pages 2241--2250, 2019.

\bibitem[Jain and Thakurta(2014)]{jain2014near}
P.~Jain and A.~G. Thakurta.
\newblock (near) dimension independent risk bounds for differentially private
  learning.
\newblock In \emph{International Conference on Machine Learning}, pages
  476--484. PMLR, 2014.

\bibitem[Jain et~al.(2012)Jain, Kothari, and Thakurta]{jain2012differentially}
P.~Jain, P.~Kothari, and A.~Thakurta.
\newblock Differentially private online learning.
\newblock In \emph{Conference on Learning Theory}, pages 24--1. JMLR Workshop
  and Conference Proceedings, 2012.

\bibitem[Kairouz et~al.(2021)Kairouz, McMahan, Song, Thakkar, Thakurta, and
  Xu]{kairouz2021practical}
P.~Kairouz, B.~McMahan, S.~Song, O.~Thakkar, A.~Thakurta, and Z.~Xu.
\newblock Practical and private (deep) learning without sampling or shuffling.
\newblock In \emph{International Conference on Machine Learning}, pages
  5213--5225. PMLR, 2021.

\bibitem[Kalai and Vempala(2005)]{kalai2005efficient}
A.~Kalai and S.~Vempala.
\newblock Efficient algorithms for online decision problems.
\newblock \emph{Journal of Computer and System Sciences}, 71\penalty0
  (3):\penalty0 291--307, 2005.

\bibitem[Kaplan et~al.(2023)Kaplan, Mansour, Moran, Nissim, and
  Stemmer]{kaplan2023differentially}
H.~Kaplan, Y.~Mansour, S.~Moran, K.~Nissim, and U.~Stemmer.
\newblock On differentially private online predictions.
\newblock \emph{arXiv preprint arXiv:2302.14099}, 2023.

\bibitem[Kleinberg and Tardos(2002)]{kleinberg2002approximation}
J.~Kleinberg and E.~Tardos.
\newblock Approximation algorithms for classification problems with pairwise
  relationships: Metric labeling and markov random fields.
\newblock \emph{Journal of the ACM (JACM)}, 49\penalty0 (5):\penalty0 616--639,
  2002.

\bibitem[Koren et~al.(2017{\natexlab{a}})Koren, Livni, and
  Mansour]{koren2017bandits}
T.~Koren, R.~Livni, and Y.~Mansour.
\newblock Bandits with movement costs and adaptive pricing.
\newblock In \emph{Conference on Learning Theory}, pages 1242--1268. PMLR,
  2017{\natexlab{a}}.

\bibitem[Koren et~al.(2017{\natexlab{b}})Koren, Livni, and
  Mansour]{koren2017multi}
T.~Koren, R.~Livni, and Y.~Mansour.
\newblock Multi-armed bandits with metric movement costs.
\newblock In \emph{Proceedings of the 31st International Conference on Neural
  Information Processing Systems}, pages 4122--4131, 2017{\natexlab{b}}.

\bibitem[Merhav et~al.(2002)Merhav, Ordentlich, Seroussi, and
  Weinberger]{merhav2002sequential}
N.~Merhav, E.~Ordentlich, G.~Seroussi, and M.~J. Weinberger.
\newblock On sequential strategies for loss functions with memory.
\newblock \emph{IEEE Transactions on Information Theory}, 48\penalty0
  (7):\penalty0 1947--1958, 2002.

\bibitem[Sherman and Koren(2021)]{shermancolt2021lazy}
U.~Sherman and T.~Koren.
\newblock Lazy oco: Online convex optimization on a switching budget.
\newblock In \emph{Conference on Learning Theory}, pages 3972--3988. PMLR,
  2021.

\bibitem[Shi et~al.(2020)Shi, Lin, Chung, Yue, and Wierman]{shi2020online}
G.~Shi, Y.~Lin, S.-J. Chung, Y.~Yue, and A.~Wierman.
\newblock Online optimization with memory and competitive control.
\newblock \emph{arXiv e-prints}, pages arXiv--2002, 2020.

\bibitem[Slivkins(2019)]{slivkins2019introduction}
A.~Slivkins.
\newblock Introduction to multi-armed bandits.
\newblock \emph{arXiv preprint arXiv:1904.07272}, 2019.

\end{thebibliography}
